\documentclass[english,12pt]{scrartcl}
\usepackage[T1]{fontenc}
\usepackage[utf8]{inputenc}
\usepackage{textcomp}
\usepackage{geometry}
\usepackage{color}
\usepackage{babel}
\usepackage{float}
\usepackage{enumitem}
\usepackage{bm}
\usepackage{amsmath}
\usepackage{amsthm}
\usepackage{amssymb}
\usepackage{graphicx}
\usepackage{esint}
\usepackage[numbers]{natbib}
\usepackage[unicode=true,pdfusetitle,bookmarks=true,bookmarksnumbered=true,bookmarksopen=true,bookmarksopenlevel=1,breaklinks=true,pdfborder={0 0 1},backref=false,colorlinks=true,linkcolor=black, citecolor=blue, urlcolor=blue, filecolor=blue,pdfpagelayout=OneColumn, pdfnewwindow=true,pdfstartview=XYZ, plainpages=false,, hyperindex=true]{hyperref} 

\makeatletter

\providecommand{\tabularnewline}{\\}

\theoremstyle{plain}

\theoremstyle{definition}
\newtheorem{defn}{\protect\definitionname}
\theoremstyle{plain}
\newtheorem{lem}{\protect\lemmaname}
\theoremstyle{plain}
\newtheorem{thm}{\protect\theoremname}
\theoremstyle{remark}
\newtheorem{rem}{\protect\remarkname}
\theoremstyle{plain}
\newtheorem{prop}{\protect\propositionname}
\theoremstyle{plain}
\newtheorem{cor}{\protect\corollaryname}
\providecommand*{\code}[1]{\texttt{#1}}

\usepackage{scrlayer}
\usepackage{lastpage}
\usepackage{amsmath}
\usepackage{amssymb}
\usepackage{graphicx}
\usepackage{xcolor}
\usepackage{fancybox} 
\usepackage{moreverb} 
\usepackage{listings} 
\usepackage{backref}
\usepackage{authblk}
\usepackage{bm}
\usepackage{array}
\usepackage{makecell}
\usepackage{multirow}
\usepackage{diagbox}

\usepackage{ifpdf} 
\ifpdf 

 \IfFileExists{lmodern.sty}{\usepackage{lmodern}}{}

\fi 

\let\myTOC\tableofcontents
\renewcommand\tableofcontents{%
  \pdfbookmark[1]{\contentsname}{}
  \myTOC
}

\def\LyX{\texorpdfstring{%
  L\kern-.1667em\lower.25em\hbox{Y}\kern-.125emX\@}
  {LyX}}

\@addtoreset{footnote}{section}

\renewcommand*{\backref}[1]{}
\renewcommand*{\backrefalt}[4]{%
   \ifcase #1 
    \or 
      (Cited on page~#2)%
   \else
      (Cited on pages~#2)
    \fi} 

\usepackage{dsfont}
\usepackage{mathtools,amsthm,amssymb,bbm}

\definecolor{blue}{HTML}{1F77B4}
\definecolor{orange}{HTML}{FF7F0E}
\definecolor{green}{HTML}{2CA02C}
\definecolor{red}{HTML}{D62728}
\definecolor{purple}{HTML}{9467BD}
\definecolor{brown}{HTML}{8C564B}
\definecolor{pink}{HTML}{E377C2}
\definecolor{grey}{HTML}{7F7F7F}
\definecolor{yellow}{HTML}{BCBD22}
\definecolor{cyan}{HTML}{17BECF}
\definecolor{turquoise}{HTML}{3FE0D0}

\usepackage[algo2e,ruled,vlined]{algorithm2e}
\usepackage{xcolor}
\definecolor{algoColorKeyword}{named}{blue}
\definecolor{algoColorComment}{named}{olive}

\SetKwComment{Comment}{$\triangleright$\ }{}

\usepackage{enumitem}
\setlist{leftmargin=*, topsep=0.5em, parsep=0pt, itemsep=1em, labelindent=0pt, align=left}

\addtokomafont{title}{}

\makeatother

\providecommand{\assumptionname}{Assumption}
\providecommand{\corollaryname}{Corollary}
\providecommand{\definitionname}{Definition}
\providecommand{\lemmaname}{Lemma}
\providecommand{\propositionname}{Proposition}
\providecommand{\remarkname}{Remark}
\providecommand{\theoremname}{Theorem}

\begin{document}
\title{A spectral mixture representation of\\ isotropic kernels with application\\ to random Fourier features}

\author[1]{Nicolas Langren\'e\thanks{Corresponding author, nicolaslangrene@bnbu.edu.cn}} 
\author[2]{Xavier Warin}
\author[2]{Pierre Gruet}

\affil[1]{\normalsize Guangdong Provincial/Zhuhai Key Laboratory of Interdisciplinary Research and Application for Data Science, Beijing Normal-Hong Kong Baptist University} 
\affil[2]{\normalsize EDF Lab Paris-Saclay, FiME (Laboratoire de Finance des March\'es de l'\'Energie)} 

\date{\vspace{-2mm}\today}

\maketitle
\vspace{-5mm}
\begin{abstract}
\citet{rahimi2007random} introduced the idea of decomposing positive definite shift-invariant kernels by randomly sampling from their spectral distribution  for machine learning applications. This famous technique, known as Random Fourier Features (RFF), is in principle applicable to any such kernel whose spectral distribution can be identified and simulated. In practice, however, it is usually applied to the Gaussian kernel because of its simplicity, since its spectral distribution is also Gaussian. Clearly, simple spectral sampling formulas would be desirable for broader classes of kernels. In this paper, we show that the spectral distribution of positive definite isotropic kernels in $\mathbb{R}^{d}$ for all $d\geq1$ can be decomposed as a scale mixture of $\alpha$-stable random vectors, and we identify the mixing distribution as a function of the kernel. This constructive decomposition provides a simple and ready-to-use spectral sampling formula for many multivariate positive definite shift-invariant kernels, including exponential power kernels, and generalized Cauchy kernels, as well as newly introduced kernels such as the generalized Mat\'ern, Tricomi, and Fox $H$ kernels. In particular, we retrieve the fact that the spectral distributions of these kernels, which can only be explicited in terms of the Fox $H$ special function, are scale mixtures of the multivariate Gaussian distribution, along with an explicit mixing distribution formula. This result has broad applications for support vector machines, kernel ridge regression, Gaussian processes, and other kernel-based machine learning techniques for which the random Fourier features technique is applicable.\\

\end{abstract}

\section{Introduction}

We start by fixing some notation and terminology regarding kernel
functions. Let $\tilde{K}:\mathbb{R}^{d}\times\mathbb{R}^{d}\rightarrow\mathbb{R}$
be a kernel function. It is said to be \textit{shift-invariant} (a.k.a.
translation-invariant, radially-symmetric, or stationary) if for any $\mathbf{x}_{i}\in\mathbb{R}^{d}$
and $\mathbf{x}_{j}\in\mathbb{R}^{d}$, $\tilde{K}(\mathbf{x}_{i},\mathbf{x}_{j})=K(\mathbf{x}_{i}-\mathbf{x}_{j})$
only depends on $\mathbf{x}_{i}$ and $\mathbf{x}_{j}$ through the
difference $\mathbf{x}_{i}-\mathbf{x}_{j}$. Moreover, the shift-invariant
kernel $K:\mathbb{R}^{d}\rightarrow\mathbb{R}$ is said to be \textit{isotropic} \citep{genton2001kernels} if it only depends on $\mathbf{x}_{i}$ and $\mathbf{x}_{j}$ through
the Euclidean norm $\left\Vert \mathbf{x}_{i}-\mathbf{x}_{j}\right\Vert $
of the difference $\mathbf{x}_{i}-\mathbf{x}_{j}$. More generally,
$K$ is said to be \textit{isotropic in the Mahalanobis distance} \citep{abe2005training}
or \textit{elliptically contoured} \citep{cambanis1981elliptically,vanmarcke2010random} if there
exists a positive definite matrix $\Sigma\in\mathbb{R}^{d\times d}$
such that $K$ only depends on $\mathbf{x}_{i}$ and $\mathbf{x}_{j}$
through the Mahalanobis distance $\sqrt{(\mathbf{x}_{i}-\mathbf{x}_{j})^{\top}\Sigma^{-1}(\mathbf{x}_{i}-\mathbf{x}_{j})}$
\citep{mahalanobis1936generalised} of the difference $\mathbf{x}_{i}-\mathbf{x}_{j}$.

Next, we say that the shift-invariant kernel $K:\mathbb{R}^{d}\rightarrow\mathbb{R}$
is \textit{positive definite} if for any $N\geq1$, $(\mathbf{x}_{1},\ldots,\mathbf{x}_{N})\in\mathbb{R}^{d\times N}$
and $(z_{1},\ldots,z_{N})\in\mathbb{R}^{N}$,
\begin{equation}
\sum_{i=1}^{N}\sum_{j=1}^{N}z_{i}z_{j}K(\mathbf{x}_{i}-\mathbf{x}_{j})\geq0.\label{eq:positive_definite-1}
\end{equation}
Denote by $\Phi_{\infty}$ the set of all continuous shift-invariant kernels
which are positive definite in $\mathbb{R}^{d}$ for all $d\geq1$.
According to Bochner's theorem \citep{bochner1933monotone,bochner1959lectures},
a continuous, shift-invariant kernel $K:\mathbb{R}^{d}\rightarrow\mathbb{R}$
is positive definite if and only if there exists a finite positive measure
$\mu$ such that
\[
K(\mathbf{u})=\intop_{\mathbb{R}^{d}}\exp(i\mathbf{x}^{\top}\mathbf{u})d\mu(\mathbf{x}),\ \ \mathbf{u}\in\mathbb{R}^{d}.
\]
In other words, $K$ is proportional to a characteristic function.
If the kernel is scaled such that $K(\mathbf{0})=1$, then $\mu$
is a probability measure. Suppose that it admits a density~$f$ (see
\citep[Theorem 1.8.16]{sasvari2013characteristic} for a characterization). Then
\begin{equation}
K(\mathbf{u})=K(\mathbf{0})\intop_{\mathbb{R}^{d}}\exp(i\mathbf{x}^{\top}\mathbf{u})f(\mathbf{x})d\mathbf{x},\ \ \mathbf{u}\in\mathbb{R}^{d}.\label{eq:multivariate_fourier_1}
\end{equation}
This shows that $K$ is the multivariate Fourier transform of $f$.
According to Bochner's theorem, $f$ is nonnegative if and only if
$K$ is positive definite. In other words, $f$ is a probability density
function, known as the \textit{spectral density} of $K$, if and only
if $K$ is positive definite.

In general, equation \eqref{eq:multivariate_fourier_1} is complex-valued.
Since a characteristic function is real-valued if and only if $f$ is symmetric
around zero (\citep[Lemma 1 page 499]{feller1971probability}, \citep[Theorem 1.3.13]{sasvari2013characteristic}), we
can further explicit equation \eqref{eq:multivariate_fourier_1} by
assuming that $f(\mathbf{x})=f(-\mathbf{x})$ for all $\mathbf{x}\in\mathbb{R}^{d}$,
in which case $K(\mathbf{u})=K(-\mathbf{u})$ for all $\mathbf{u}\in\mathbb{R}^{d}$
and 
\begin{equation}
K(\mathbf{u})=K(\mathbf{0})\intop_{\mathbb{R}^{d}}\cos(\mathbf{x}^{\top}\mathbf{u})f(\mathbf{x})d\mathbf{x},\ \ \mathbf{u}\in\mathbb{R}^{d}.\label{eq:K_wrt_f}
\end{equation}
Moreover, $K$ and $f$ are Fourier duals:
\begin{align}
f(\mathbf{x}) & =\frac{1}{K(\mathbf{0})(2\pi)^{d}}\intop_{\mathbb{R}^{d}}\cos(\mathbf{x}^{\top}\mathbf{u})K(\mathbf{u})d\mathbf{u},\ \ \mathbf{x}\in\mathbb{R}^{d}.\label{eq:f_wrt_K}
\end{align}
The interesting aspect of equation \eqref{eq:K_wrt_f} is that if
$K$ is positive definite, then $f$ is a probability
density function, and the following probabilistic representation holds
\begin{align}
K(\mathbf{u}) & =K(\mathbf{0})\mathbb{E}\left[\cos(\bm{\eta}^{\top}\mathbf{u})\right]\label{eq:random_projections}
\end{align}
where $\bm{\eta}=(\eta_{1},\ldots,\eta_{d})$ is a continuous random
vector with density $f$. The random vector $\bm{\eta}$ is known
as \textit{random projection}. The probabilistic representation \eqref{eq:random_projections}
suggests that the kernel $K$ can be approximated by Monte Carlo simulations:
\begin{equation}
K(\mathbf{u})\simeq K_{M}(\mathbf{u}):=\frac{K(\mathbf{0})}{M}\sum_{m=1}^{M}\cos(\bm{\eta}_{m}^{\top}\mathbf{u}),\ \ \mathbf{u}\in\mathbb{R}^{d}.\label{eq:random_fourier_features}
\end{equation}
This approach is known as \textit{random Fourier features} \citep{rahimi2007random}
in machine learning. While a direct evaluation of all the kernel values
$K(\mathbf{x}_{i}-\mathbf{x}_{j})$ on a dataset $\{\mathbf{x}_{1},\ldots,\mathbf{x}_{N}\}$
would cost a prohibitive $\mathcal{O}(N^{2})$ operations on large datasets,
equation \eqref{eq:random_fourier_features} reduces this computational
cost to $\mathcal{O}(MN)$, where $M$ is the number of spectral simulations
in \eqref{eq:random_fourier_features}, thanks to the following explicit
feature mapping:
\begin{align}
 & K_{M}(\mathbf{x}_{i}-\mathbf{x}_{j})=\varphi(\mathbf{x}_{i})^{\top}\varphi(\mathbf{x}_{j}),\ \mathrm{where,\ for\ any\ }\mathbf{x}\in\mathbb{R}^{d},\label{eq:feature_mapping_1}\\
 & \varphi(\mathbf{x}):=\frac{\sqrt{K(\mathbf{0})}}{\sqrt{M}}\left[\begin{array}{cccccc}
\!\cos(\bm{\eta}_{1}^{\top}\mathbf{x})\! & \!\ldots\! & \!\cos(\bm{\eta}_{M}^{\top}\mathbf{x})\! & \!\sin(\bm{\eta}_{1}^{\top}\mathbf{x})\! & \!\ldots\! & \!\sin(\bm{\eta}_{M}^{\top}\mathbf{x})\!\end{array}\right]^{\top}\!\in\mathbb{R}^{2M}\label{eq:feature_mapping_2}
\end{align}
This efficient approximation has broad applications in machine learning
and statistical learning, for example for algorithms such as kernel
ridge regression, Gaussian process inference, kernel principal component
analysis, support vector machines, and other kernel-based methods
\citep{francis2021major}.

It can be noted that the idea of spectral sampling predates \citep{rahimi2007random}.
In particular, \citep{shinozuka1972monte,shinozuka1972digital} proposed
to simulate from the spectral density of a covariance function in
order to efficiently simulate non-conditional Gaussian random fields.
This spectral representation method \citep{deodatis2025spectral}
for Gaussian simulations, and the closely related spectral turning bands method \citep{mantoglou1982turning}, have been applied in a number of fields, most
notably mechanical and structural engineering \citep{shinozuka1991simulation}
and geostatistics \citep{lantuejoul1994nonconditional,lantuejoul2002geostatistical}.
That being said, \citep{rahimi2007random} should still be credited
for introducing spectral sampling to the field of machine learning.

In order to implement equation \eqref{eq:random_fourier_features},
the spectral distribution of $K$ needs to be precomputed. \citet{rahimi2007random}
provide three examples of shift-invariant multivariate kernels amenable
to this method: the Gaussian kernel, the Laplace kernel and the Cauchy
kernel:
\begin{itemize}
\item The Gaussian kernel admits a Gaussian spectral density
\begin{eqnarray}
K(\mathbf{u})=\exp\left(-\frac{1}{2}\left\Vert \mathbf{u}\right\Vert ^{2}\right) & \ ,\  & f(\mathbf{x})=\frac{1}{(2\pi)^{d/2}}\exp\left(-\frac{1}{2}\left\Vert \mathbf{x}\right\Vert ^{2}\right)\ ,\label{eq:gaussian_kernel}
\end{eqnarray}
\item The Laplace kernel admits a Cauchy spectral density
\begin{eqnarray}
K(\mathbf{u})=\exp\left(-\sum_{\ell=1}^{d}\left|u_{\ell}\right|\right) & \ ,\  & f(\mathbf{x})=\frac{1}{\pi^{d}}\prod_{\ell=1}^{d}\frac{1}{1+x_{\ell}^{2}}\ ,\label{eq:laplace_kernel}
\end{eqnarray}
\item The Cauchy kernel admits a Laplace spectral density
\begin{eqnarray}
K(\mathbf{u})=\prod_{\ell=1}^{d}\frac{1}{1+u_{\ell}^{2}} & \ ,\  & f(\mathbf{x})=\frac{1}{2^{d}}\exp\left(-\sum_{\ell=1}^{d}\left|x_{\ell}\right|\right).\label{eq:cauchy_kernel}
\end{eqnarray}
\end{itemize}
where $\mathbf{u}=(u_{1},\ldots,u_{d})\in\mathbb{R}^{d}$ and $\mathbf{x}=(x_{1},\ldots,x_{d})\in\mathbb{R}^{d}$.
Remark that the two kernels \eqref{eq:laplace_kernel}-\eqref{eq:cauchy_kernel}
are not isotropic kernels but tensor kernels, obtained as the product
of univariate kernels; for the Gaussian kernel \eqref{eq:gaussian_kernel},
the isotropic and tensor formulations coincide. Remark also how the roles of $K$ and $f$
can be swapped: if $K$ is a nonnegative integrable kernel with spectral
density $f$, then $f/f(\mathbf{0})$ is a kernel with spectral density
$K/((2\pi)^{d}f(\mathbf{0}))$. This is a particular case of the duality
theorem in \citet{harrar2006duality}; see also \citep{gneiting1997mixture}.

In principle, the random Fourier features formula \eqref{eq:random_fourier_features}
can be applied to any positive definite shift-invariant kernel for
which the spectral density can be computed and simulated. In practice,
the simplicity and convenience of the three examples \eqref{eq:gaussian_kernel}-\eqref{eq:laplace_kernel}-\eqref{eq:cauchy_kernel},
especially the Gaussian kernel \eqref{eq:gaussian_kernel}, means
that there is not much work available on the application of random
Fourier features for other multivariate parametric kernels in machine
learning. In the context of Gaussian random field simulation, spectral
sampling has been applied to a few more kernel covariance functions
\citep{emery2006tbsim}. In $\Phi_{\infty}$, the most notable one
is the Mat\'ern kernel $K(\mathbf{u})=\frac{(\sqrt{2\nu}\left\Vert \mathbf{u}\right\Vert )^{\nu}}{\Gamma(\nu)2^{\nu-1}}\mathcal{K}_{\nu}(\sqrt{2\nu}\left\Vert \mathbf{u}\right\Vert )$,
$\nu>0$, \citep{matern1960spatial}, where $\mathcal{K}_{\beta}$ is the modified Bessel function \citep[10.25]{dlmf}, whose spectral density is a multivariate
Student $t$-distribution \citep{emery2016improved,chen2022kernel}.
In order to consider more general kernel functions, one clear challenge
is the ability to compute the multivariate inverse Fourier transform
\eqref{eq:f_wrt_K} analytically, and to find a suitable sampling
algorithm for the corresponding distribution.

This is the task that we propose to address in this paper, with a
focus on isotropic kernels which are positive definite in $\mathbb{R}^{d}$
for all $d\geq1$ ($K\in\Phi_{\infty}$). Our article makes the following
contributions:
\begin{itemize}
\item We notice that the random projections of kernels in $\Phi_{\infty}$
are necessarily Gaussian scale mixtures, since Schoenberg's theorem
characterizes kernels in $\Phi_{\infty}$ as Gaussian scale mixtures
\citep[Theorem 3.8.5]{sasvari2013characteristic} and that the Fourier
transform of a Gaussian scale mixture is also a Gaussian scale mixture
\citep{gneiting1997mixture}. For practical purposes, this means that
multiplying multivariate Gaussian simulations by simulations from
a specific nonnegative distribution conveniently extends the random
Fourier features methodology from Gaussian kernels to any positive
definite isotropic kernel in $\mathbb{R}^{d}$ for all $d\geq1$ (same
mixture simulations in every dimension) or tensor kernel (independent
mixture simulations in each dimension) of interest (Remark~\ref{rem:tensor_kernel}).
\item We prove that if a multivariate isotropic kernel $K(\mathbf{u})=k(\left\Vert \mathbf{u}\right\Vert ^{2})$
belongs to $\Phi_{\infty}$, then all the kernels $K_{\alpha}(\mathbf{u}):=k(\left\Vert \mathbf{u}\right\Vert ^{\alpha})$,
$\alpha\in(0,2]$, also belong to $\Phi_{\infty}$ (Theorem~\ref{thm:isotropic_random_projection}).
In other words, the generalization treatment that was applied to the
isotropic Cauchy kernel in \citep{gneiting2004stochastic} actually works for
every kernel in $\Phi_{\infty}$. This includes such kernels as the
Mat\'ern kernel \citep{matern1960spatial} and the confluent hypergeometric
kernel \citep{ma2023beyond}, whose generalized version can be found
in Table~\ref{tab:new_random_projections}.
\item Using results from the theories of multivariate stable distributions
\citep{samorodnitsky1994stable,devroye2014simulation,nolan2020stable}
and multivariate characteristic functions \citep{sasvari2013characteristic},
we prove that the spectral distributions of these generalized $\Phi_{\infty}$
kernels $k(\left\Vert \mathbf{u}\right\Vert ^{\alpha})$, $\alpha\in(0,2]$,
can conveniently be expressed as scale mixtures of multivariate symmetric
stable distributions. The scaling distribution is explicitly given
by the inverse Laplace transform of the kernel function $k$ (Theorem~\ref{thm:isotropic_random_projection},
Corollary~\ref{cor:isotropic_random_projection}). We provide explicit
formulas to simulate these spectral distributions using known results
about stable distributions \citep{chambers1976method,nolan2020stable}. 
\item We apply our scale mixture decomposition of random projections to
several shift-invariant kernels of interest in machine learning, such
as exponential power kernels, Mat\'ern kernels, and generalized Cauchy
kernels, and use it to construct new multivariate isotropic kernels
in $\Phi_{\infty}$. By defining the nonnegative spectral radius random
variable $R$ as a transform of one or two Gamma random variables,
we created new kernel functions such as the Beta, Kummer, Tricomi,
and Fox $H$ kernels (Table~\ref{tab:new_random_projections}). The
availability of new kernel functions with additional parameters is
of great practical interest, especially in the situation when classical
parametric kernels are too rigid to properly capture the information
available in the data.
\item Following \citep{schneider1986stable}, we show that the density of
symmetric stable vectors can be expressed in terms of the Fox $H$
generalized hypergeometric special function (Proposition~\ref{prop:fourier_exponential_power}).
Since the class of $H$-functions is stable by mixing \citep[equation~2.25.1]{prudnikov1990integrals},
we obtain that the spectral densities of all the kernels mentioned
in this article can be expressed analytically in terms of the Fox
$H$-function, with several explicit examples given in Appendix~\ref{sec:spectral_densities}.
These spectral density formulas are of interest in their own right;
that being said, for the practical implementation of spectral sampling,
our spectral mixture representation of these spectral densities (Theorem~\ref{thm:isotropic_random_projection})
is more convenient (Section~\ref{sec:sampling}).
\item We implement and illustrate these kernels along with their random
Fourier features approximations, implemented using our spectral mixture
representation. Our numerical experiments confirm the exactness of
our theoretical results. 
\end{itemize}
The rest of the paper is organized as follows. Section \ref{sec:spectral_mixture}
describes our characterization of random projections as scale mixtures,
and applies it to several examples of multivariate isotropic kernels.
Section \ref{sec:sampling} describes how to sample such random projections
in practice. Section \ref{sec:numerical} provides numerical examples,
and finally Section \ref{sec:conclusion} concludes the paper. Appendix~\ref{sec:characteristic_functions}
describes new mixing distributions, and Appendix~\ref{sec:spectral_densities}
provides analytical expressions for the spectral densities of several
kernels considered in this paper.

\section{Mixture representation of spectral distributions\label{sec:spectral_mixture}}

In this section, we show that the spectral distribution of (multivariate)
isotropic kernels in $\Phi_\infty$ can always be characterized as a scale mixture of
symmetric stable distributions (Theorem~\ref{thm:isotropic_random_projection}).
We characterize this scaling distribution and use it to decompose
the random projections of several multivariate isotropic kernels of
interest, and easily create new ones (Table~\ref{tab:new_random_projections}). To do so, we need the following definition.
\begin{defn}
\label{def:symmetric_stable}For any $\alpha\in(0,2]$, let $\boldsymbol{S}_{\alpha}$
be a $d$-dimensional random vector with characteristic function $\phi_{\alpha}$
given by
\begin{equation}
\phi_{\alpha}(\mathbf{u})=\mathbb{E}\left[e^{i\boldsymbol{S}_{\alpha}^{\top}\mathbf{u}}\right]=e^{-\left\Vert \mathbf{u}\right\Vert ^{\alpha}}\ ,\ \mathbf{u}\in\mathbb{R}^{d}\label{eq:symmetric_stable-1}
\end{equation}
where $\left\Vert \mathbf{u}\right\Vert =\sqrt{u_{1}^{2}+\ldots+u_{d}^{2}}$
is the Euclidean norm of $\mathbf{u}=(u_{1},\ldots,u_{d})\in\mathbb{R}^{d}$.
The vector $\boldsymbol{S}_{\alpha}$ is called \textit{symmetric
stable} (see for example \citep{samorodnitsky1994stable}).
\end{defn}
\begin{lem}
\label{lem:isotropic_random_projection}For any $\alpha\in(0,2]$,
let $\boldsymbol{S}_{\alpha}$ be a $d$-dimensional symmetric stable
random vector (Definition~\ref{def:symmetric_stable}), let $R$
be a real-valued nonnegative random variable, independent of $\boldsymbol{S}_{\alpha}$,
with Laplace transform $\mathcal{L}$, and let $\boldsymbol{\Sigma}\in\mathbb{R}^{d\times d}$
be a symmetric, positive definite matrix. Then, the random projection
vector defined by
\begin{equation}
\boldsymbol{\eta}=R^{\frac{1}{\alpha}}\boldsymbol{\Sigma}^{-\frac{1}{2}}\boldsymbol{S}_{\alpha}\label{eq:isotropic_random_projection}
\end{equation}
spans the following isotropic kernel $K:\mathbb{R}^{d}\rightarrow\mathbb{R}$:
\begin{equation}
K(\mathbf{u})=K(\mathbf{0})\mathbb{E}\left[e^{i\boldsymbol{\eta}^{\top}\mathbf{u}}\right]=K(\mathbf{0})\mathcal{\mathcal{L}}(\left\Vert \mathbf{u}\right\Vert _{\boldsymbol{\Sigma}}^{\alpha})\ ,\ \mathbf{u}\in\mathbb{R}^{d}.\label{eq:isotropic_kernel}
\end{equation}
where $\left\Vert \mathbf{u}\right\Vert _{\boldsymbol{\Sigma}}:=\sqrt{\mathbf{u}^{T}\boldsymbol{\Sigma}^{-1}\mathbf{u}}$
is the Mahalanobis norm of $\mathbf{u}$ with respect to the covariance
matrix $\boldsymbol{\Sigma}$.
\end{lem}
\begin{proof}
Recall that the Laplace transform $\mathcal{L}$ of a nonnegative
random variable $R$ is defined by $\mathcal{L}(s)=\mathbb{E}\left[e^{-sR}\right]$
for $s\geq0$. Then, the characteristic function of $\bm{\eta}$ is
given by
\begin{align*}
\mathbb{E}\left[e^{i\boldsymbol{\eta}^{\top}\mathbf{u}}\right] & =\mathbb{E}\left[\mathbb{E}\left[\exp\left(i\boldsymbol{S}_{\alpha}^{\top}\left(R^{\frac{1}{\alpha}}\boldsymbol{\Sigma}^{-\frac{1}{2}}\mathbf{u}\right)\right)\left|R\right.\right]\right]\\
 & =\mathbb{E}\left[\exp\left(-R\left\Vert \boldsymbol{\Sigma}^{-\frac{1}{2}}\mathbf{u}\right\Vert ^{\alpha}\right)\right]\\
 & =\mathcal{\mathcal{L}}(\left\Vert \mathbf{u}\right\Vert _{\boldsymbol{\Sigma}}^{\alpha})
\end{align*}
which proves equation \eqref{eq:isotropic_kernel}.
\end{proof}
Lemma~\ref{lem:isotropic_random_projection} has previously been
used for example in \citep{pakes1998mixture} in the case $d=1$ and
identity matrix $\boldsymbol{\Sigma}=\boldsymbol{I}$. 
Now, the following Theorem~\ref{thm:isotropic_random_projection}
provides a characterization of the spectral distribution of continuous,
positive definite, isotropic kernels in $\Phi_{\infty}$ as a scale mixture of $\alpha$-stable
random vectors. Theorem~\ref{thm:isotropic_random_projection} sets
$\boldsymbol{\Sigma}=\boldsymbol{I}$ for simplicity, but its results
still hold for general covariance matrix $\boldsymbol{\Sigma}$ (Remark~\ref{rem:mahalanobis_kernel}).
\begin{thm}
\label{thm:isotropic_random_projection}Let $K(\mathbf{u})=k(\left\Vert \mathbf{u}\right\Vert ^{2})$,
$\mathbf{u}\in\mathbb{R}^{d}$, be a kernel function in $\Phi_{\infty}$. Then 

a) there exists a nonnegative random variable $R$ such that its Laplace
transform is equal to $k/k(0)$, 

b) for every $\alpha\in(0,2]$, the isotropic kernel $K_{\alpha}(\mathbf{u}):=k(\left\Vert \mathbf{u}\right\Vert ^{\alpha})$
is also positive definite in $\mathbb{R}^{d}$ for every $d\geq1$, 

c) for every $\alpha\in(0,2]$, the unique random projection vector $\boldsymbol{\eta}_{\alpha}$
of the kernel $K_{\alpha}$ admits the representation $\boldsymbol{\eta}_{\alpha}=R^{\frac{1}{\alpha}}\boldsymbol{S}_{\alpha}$, where $\boldsymbol{S}_{\alpha}$ is a symmetric stable random vector
independent of $R$.
\end{thm}
\begin{proof}
Part a) is a rephrasing of Schoenberg's theorem for kernel functions in $\Phi_{\infty}$ \citep[Theorem~3.8.5]{sasvari2013characteristic}
and the Hausdorff-Bernstein-Widder theorem \citep[Theorem~3.9.6]{sasvari2013characteristic},
which we need to introduce the random variable $R\geq0$ whose existence
is needed in parts b) and c). As clearly stated in \citep[Theorem~3.9.8]{sasvari2013characteristic},
the fact that the isotropic kernel $K(\mathbf{u})=k(\left\Vert \mathbf{u}\right\Vert ^{2})$,
$\mathbf{u}\in\mathbb{R}^{d}$, is continuous and positive definite
for every $d\geq1$ is equivalent to the fact that there exists a
finite nonnegative measure $\mu$ on $[0,\infty)$ such that $k(u)=\int_{0}^{\infty}e^{-us}d\mu(s)$
for all $u\in[0,\infty)$. This means that $k(0)$ is finite and there
exists a nonnegative random variable $R\geq0$ such that $k(u)=k(0)\mathbb{E}\left[e^{-uR}\right]$.
In other words, $k/k(0)$ is the Laplace transform of $R$. 

Then, using Lemma~\ref{lem:isotropic_random_projection} with this random variable $R$ and $\boldsymbol{\Sigma}=\boldsymbol{I}$, the random projection vector $\boldsymbol{\eta}_{\alpha}=R^{\frac{1}{\alpha}}\boldsymbol{S}_{\alpha}$
spans the isotropic kernel $K_{\alpha}(\mathbf{u})=k(\left\Vert \mathbf{u}\right\Vert ^{\alpha})$,
where we used the fact that $K(\mathbf{0})=k(0)$. The existence of
$\boldsymbol{\eta}_{\alpha}$ proves b) by Bochner's theorem, and
the Uniqueness Theorem~1.3.3 in \citep{sasvari2013characteristic}
completes the proof of c).
\end{proof}
It is sometimes more convenient to work with a kernel defined as $k(\left\Vert \mathbf{u}\right\Vert )$
(as in the definition of isotropic kernels) instead of $k(\left\Vert \mathbf{u}\right\Vert ^{2})$.
A simple change of variable in Theorem~\ref{thm:isotropic_random_projection}
gives the following Corollary:
\begin{cor}\label{cor:isotropic_random_projection}Let $K(\mathbf{u})=k(\left\Vert \mathbf{u}\right\Vert )$,
$\mathbf{u}\in\mathbb{R}^{d}$, be a kernel function in $\Phi_{\infty}$. Then 

a) there exists a nonnegative random variable $R$ such that its Laplace
transform is equal to $k(\sqrt{.})/k(0)$, 

b) for every $\alpha\in(0,1]$, the isotropic kernel $K_{\alpha}(\mathbf{u}):=k(\left\Vert \mathbf{u}\right\Vert ^{\alpha})$
is also positive definite in $\mathbb{R}^{d}$ for every $d\geq1$, 

c) for every $\alpha\in(0,1]$, the unique random projection vector $\boldsymbol{\eta}_{\alpha}$
of the kernel $K_{\alpha}$ admits the representation $\boldsymbol{\eta}_{\alpha}=R^{\frac{1}{2\alpha}}\boldsymbol{S}_{2\alpha}$, where $\boldsymbol{S}_{2\alpha}$ is a symmetric stable random vector
independent of $R$.
\end{cor}
In the case $\alpha=2$ in Theorem~\ref{thm:isotropic_random_projection},
one can easily retrieve the fact that the random projection is a scale mixture of
a Gaussian vector:
\begin{cor}
\label{cor:alpha_2}In the case $\alpha=2$, the continuous, positive
definite, isotropic kernel $K(\mathbf{u})=k(\left\Vert \mathbf{u}\right\Vert ^{2})$
admits the random projection vector $\boldsymbol{\eta}=\sqrt{2R}\boldsymbol{N}$,
where the distribution of $R$ is the inverse Laplace transform of $k/k(0)$, and $\boldsymbol{N}$
is a $d$-dimensional standard Gaussian vector independent of $R$.
\end{cor}
\begin{proof}
Apply Theorem~\ref{thm:isotropic_random_projection} with $\alpha=2$
and use the fact that $\boldsymbol{S}_{2}\overset{d}{=}\sqrt{2}\boldsymbol{N}$
where $\boldsymbol{N}$ is a $d$-dimensional standard Gaussian vector
(from Definition~\ref{def:symmetric_stable}).
\end{proof}
\begin{rem}
\label{rem:mahalanobis_kernel}The results of Theorem~\ref{thm:isotropic_random_projection}
still hold if the Euclidean norm $\left\Vert \mathbf{u}\right\Vert =\sqrt{\mathbf{u}^{\top}\mathbf{u}}$
is replaced by the Mahalanobis norm $\left\Vert \mathbf{u}\right\Vert _{\boldsymbol{\Sigma}}=\sqrt{\mathbf{u}^{\top}\Sigma^{-1}\mathbf{u}}$
where $\boldsymbol{\Sigma}\in\mathbb{R}^{d\times d}$ is a symmetric,
positive definite matrix, in which case the random projection vector
is given by $\boldsymbol{\eta}=R^{\frac{1}{\alpha}}\boldsymbol{\Sigma}^{-\frac{1}{2}}\boldsymbol{S}_{\alpha}$.
This is another possible approach to introduce additional parameters
into the formula of a multivariate parametric kernel.
\end{rem}
\begin{rem}
Corollary~\ref{cor:alpha_2} shows that the random projection vector
of the continuous, positive definite, isotropic kernel $K(\mathbf{u})=k(\left\Vert \mathbf{u}\right\Vert ^{2})$
is a scale mixture of a standard Gaussian vector. In fact, the same
is true for all the kernels $K_{\alpha}(\mathbf{u})=k(\left\Vert \mathbf{u}\right\Vert ^{\alpha})$,
$\alpha\in(0,2]$, as pointed out later in Proposition~\ref{prop:multi_exp_power}
(see Corollary~\ref{cor:scale_mixture_gaussian}).
\end{rem}
\begin{rem}
\label{rem:characteristic_function}The distribution of the nonnegative random variable
$R$ in Theorem~\ref{lem:isotropic_random_projection} is the inverse
Laplace transform of the scaled kernel $k/k(0)$. In other words,
the Laplace transform $\mathcal{L}$ of $R$ is equal to $\mathcal{L}(s)=k(s)/k(0)$
for every $s\geq0$. Equivalently, $R$ can be characterised by its
characteristic function $\phi(s)=\mathbb{E}\left[e^{isR}\right]$.
Indeed, for a nonnegative random variable the equality $\mathcal{\mathcal{L}}(s)=\phi(is)$
holds for every $s\geq0$. This means that for every $\alpha\in(0,2]$,
the continuous, positive definite, isotropic kernel $K_{\alpha}(\mathbf{u})=k(\left\Vert \mathbf{u}\right\Vert ^{\alpha})$
can be written as
\begin{equation}
K_{\alpha}(\mathbf{u})=k(\left\Vert \mathbf{u}\right\Vert ^{\alpha})=k(0)\mathcal{\mathcal{L}}(\left\Vert \mathbf{u}\right\Vert ^{\alpha})=k(0)\phi(i\left\Vert \mathbf{u}\right\Vert ^{\alpha})\label{eq:laplace_characteristic}
\end{equation}
for every $\mathbf{u}\in\mathbb{R}^{d}$. The characteristic function formulation \eqref{eq:laplace_characteristic}
is more convenient to construct new positive definite multivariate
isotropic kernels from a given nonnegative distribution for $R$,
since characteristic functions are much more commonly precomputed
and available for a vast range of positive distributions \citep{oberhettinger1973fourier,oberhettinger1990fourier}.
This is the approach adopted in Table~\ref{tab:new_random_projections} with a scaling factor $\lambda>0$, except for the last two Fox $H$-kernels, for which we first evaluated the Laplace transform $\mathcal{L}$ directly in Appendix~\ref{sec:characteristic_functions}.
\end{rem}
In the following, we provide several examples of interest for which
both the kernel $K$ and the distribution of the random variable $R\geq0$
from Theorem~\ref{thm:isotropic_random_projection} are known analytically.
We also use the approach described in Remark~\ref{rem:characteristic_function}
to create new multivariate positive definite kernels. These examples are listed in Table~\ref{tab:new_random_projections}. Other valid
kernel examples can be deduced for example from \citep{gneiting1997mixture} and \citep{mcnolty1973characteristic}, for continuous and discrete spectral mixing distributions respectively. Table~\ref{tab:new_random_projections}
makes use of the following special functions:
\begin{itemize}
\item $\Gamma$ is the gamma function \citep[5.2]{dlmf}
\item $\mathcal{B}(a,b)$ is the beta function \citep[5.12]{dlmf}
\item $\mathcal{K}_{\beta}$ is the modified Bessel function \citep[10.25]{dlmf},
\item $\mathcal{M}(a,b,z)$ is the Kummer confluent hypergeometric function
\citep[13.2]{dlmf}, also denoted as $_{1}F_{1}(a,b,z)$ in some references,
\item $\mathcal{U}(a,b,z)$ is the Tricomi confluent hypergeometric function,
a.k.a. Kummer's function of the second kind \citep[13.2]{dlmf}.
\item $H_{p,q}^{m,n}\bigg(z\bigg|\begin{array}{l}
{\scriptstyle (a_{1},\alpha_{1}),\ldots,(a_{p},\alpha_{p})}\\
{\scriptstyle (b_{1},\beta_{1}),\ldots,(b_{q},\beta_{q})}
\end{array}\bigg)$ is the Fox $H$-function \citep{fox1961hfunctions,mathai2010hfunction,coelho2019finite}.
\end{itemize}

Table~\ref{tab:new_random_projections} also uses the following random
variables:
\begin{itemize}
\item $\boldsymbol{N}$ is a standard multivariate Gaussian random vector,
\item $G_{\beta}$, $\beta>0$, is a Gamma random variable, with density
$f(x)=\frac{1}{\Gamma(\beta)}x^{\beta-1}e^{-x}\,\mathbbm{1}_{\{x>0\}}$,
\item $B_{\beta,\gamma}$, $\beta>0$, $\gamma>0$, is a Beta random variable,
with density given by $f(x)=\frac{x^{\beta-1}(1-x)^{\gamma-1}}{\mathcal{B}(\beta,\gamma)}\mathbbm{1}_{\{0<x<1\}}$. It can be obtained from two independent Gamma random
variables $G_{\beta}$ and $G_{\gamma}$ as $B_{\beta,\gamma}\overset{d}{=}\frac{G_{\beta}}{G_{\beta}+G_{\gamma}}$,
\item $F_{2\beta,2\gamma}$, $\beta>0$, $\gamma>0$, is a Fisher-Snedecor
random variable, with density $f(x)=\frac{1}{x\mathcal{B}(\beta,\gamma)}\frac{(\beta x)^{\beta}\gamma^{\gamma}}{(\beta x+\gamma)^{\beta+\gamma}}\mathbbm{1}_{\{x>0\}}$. It can be obtained from two independent Gamma random variables
$G_{\beta}$ and $G_{\gamma}$ as $F_{2\beta,2\gamma}\overset{d}{=}\frac{\gamma G_{\beta}}{\beta G_{\gamma}}$. It is a rescaling of the beta prime distribution \citep[equation~(13.1)]{crooks2019field}.
\end{itemize}

\begin{table}[H]
\renewcommand{\arraystretch}{1.3} 
\begin{centering}
\begin{tabular}{llll}
\hline 
 & Name & Formula & Scaling \tabularnewline
\hline 
 &  &  & \tabularnewline
$R=\ $ & Constant $1$ & $\phi(x)=e^{ix}$ & \tabularnewline
$K=\ $ & Exponential power & $K(\mathbf{u})=e^{-\left\Vert \mathbf{u}\right\Vert ^{\alpha}}$ & $\lambda=1$\tabularnewline
 &  &  & \tabularnewline
$R=\ $ & Gamma $ G_\beta\ ,\ \beta>0$ & $\phi(x)=(1-ix)^{-\beta}$ & \tabularnewline
$K=\ $ & Generalized Cauchy & $K(\mathbf{u})=\frac{1}{\left(1+\frac{\left\Vert \mathbf{u}\right\Vert ^{\alpha}}{2\beta}\right)^{\beta}}$ & $\lambda=\frac{1}{2\beta}$\tabularnewline
 &  &  & \tabularnewline
$R=\ $ & \makecell[l]{Inverse Gamma\\ $1/ G_\beta ,\ \beta>0$}  & $\phi(x)=\frac{2(-ix)^{\beta/2}}{\Gamma(\beta)}\mathcal{K}_{\beta}(\sqrt{-4ix})$ & \tabularnewline
$K=\ $ & Generalized Mat\'ern & $K(\mathbf{u})=\frac{(\sqrt{2\beta}\left\Vert \mathbf{u}\right\Vert ^{\frac{\alpha}{2}})^{\beta}}{\Gamma(\beta)2^{\beta-1}}\mathcal{K}_{\beta}(\sqrt{2\beta}\left\Vert \mathbf{u}\right\Vert ^{\frac{\alpha}{2}})$ & $\lambda=\frac{\beta}{2}$\tabularnewline
 &  &  & \tabularnewline
$R=\ $ & \makecell[l]{Beta\\ $B_{\beta,\gamma}\ ,\ \beta>0,\ \gamma>0$} & $\phi(x)=\mathcal{M}(\beta,\beta+\gamma,ix)$ & \tabularnewline
$K=\ $ & Kummer & $K(\mathbf{u})=\mathcal{M}(\beta,\beta+\gamma,-\left\Vert \mathbf{u}\right\Vert ^{\alpha})$ & $\lambda=1$\tabularnewline
 &  &  & \tabularnewline
$R=\ $ & \makecell[l]{Beta-exponential\\ $-\log(B_{\beta,\gamma}),\beta>0,\gamma>0$} & $\phi(x)=\frac{\mathcal{B}(\beta-ix,\gamma)}{\mathcal{B}(\beta,\gamma)}$ & \tabularnewline
$K=\ $ & Beta & $K(\mathbf{u})=\frac{\mathcal{B}(\beta+\left\Vert \mathbf{u}\right\Vert ^{\alpha},\gamma)}{\mathcal{B}(\beta,\gamma)}$ & $\lambda=1$\tabularnewline
 &  &  & \tabularnewline
$R=\ $ & \makecell[l]{$F$-distribution\\ $F_{2\beta,2\gamma}\ ,\ \beta>0, \gamma>0$} & $\phi(x)=\frac{\Gamma\left(\beta+\gamma\right)}{\Gamma\left(\gamma\right)}\,\mathcal{U}\!\left(\beta,1-\gamma,-i\frac{\gamma}{\beta}x\right)$ & \tabularnewline
$K=\ $ & Tricomi & $K(\mathbf{u})=\frac{\Gamma\left(\beta+\gamma\right)}{\Gamma\left(\gamma\right)}\,\mathcal{U}\!\left(\beta,1-\gamma,\frac{\gamma}{\beta}\left\Vert \mathbf{u}\right\Vert ^{\alpha}\right)$ & $\lambda=1$\tabularnewline
 &  &  & \tabularnewline
$R=\ $ & \makecell[l]{Stacy $G_{\beta}^{\frac{1}{\ell}}\ ,\ \beta>0, \ell>0$} & $\phi(x)=\frac{1}{\Gamma(\beta)}H_{1,1}^{1,1}\Big(-ix\,\Big|\begin{array}{l}
{\scriptstyle (1-\beta,\frac{1}{\ell})}\\
{\scriptstyle (0,1)}
\end{array}\Big)$ & \tabularnewline
$K=\ $ & Fox $H_{1,1}^{1,1}$ & $K(\mathbf{u})=\frac{1}{\Gamma(\beta)}H_{1,1}^{1,1}\Big(\left\Vert \mathbf{u}\right\Vert ^{\alpha}\Big|\begin{array}{l}
{\scriptstyle (1-\beta,\frac{1}{\ell})}\\
{\scriptstyle (0,1)}
\end{array}\Big)$ & $\lambda=1$\tabularnewline
 &  &  & \tabularnewline
$R=\ $ & \makecell[l]{Generalized beta prime\\ $(G_{\beta}/G_{\gamma})^\frac{1}{\ell}\ ,\ \beta, \gamma, \ell>0$} & $\phi(x)=\frac{1}{\Gamma(\beta)\Gamma(\gamma)}H_{1,2}^{2,1}\Big(-ix\,\Big|\begin{array}{l}
{\scriptstyle (1-\beta,\frac{1}{\ell})}\\
{\scriptstyle (0,1),(\gamma,\frac{1}{\ell})}
\end{array}\Big)$ & \tabularnewline
$K=\ $ & Fox $H_{1,2}^{2,1}$ & $K(\mathbf{u})=\frac{1}{\Gamma(\beta)\Gamma(\gamma)}H_{1,2}^{2,1}\Big(\left\Vert \mathbf{u}\right\Vert ^{\alpha}\Big|\begin{array}{l}
{\scriptstyle (1-\beta,\frac{1}{\ell})}\\
{\scriptstyle (0,1),(\gamma,\frac{1}{\ell})}
\end{array}\Big)$ & $\lambda=1$\tabularnewline
 &  &  & \tabularnewline
\hline 
\end{tabular}
\par\end{centering}
\caption{Characteristic functions $\phi$ of random mixture distributions $R\geq0$ in the random projection formula $\boldsymbol{\eta}=(\lambda R)^{\frac{1}{\alpha}}\boldsymbol{S}_{\alpha}$, scaled with $\lambda>0$, and resulting covariance kernels $K(\mathbf{u})=K(\mathbf{0})\phi(i\lambda\left\Vert \mathbf{u}\right\Vert ^{\alpha})$.\label{tab:new_random_projections}}
\end{table}

Particular cases of covariance kernels of interest from Table~\ref{tab:new_random_projections}
include:
\begin{itemize}
\item The Laplace kernel $K(\mathbf{u})=e^{-\left\Vert \mathbf{u}\right\Vert }$
($R=1$, $\lambda=1$, $\alpha=1$), also known as exponential kernel
\citep{rasmussen2006gaussian} (see for example \citep[equation~(4)]{francis2021major}). The random projection $\boldsymbol{\eta}=\boldsymbol{S}_{1}$
follows a standard multivariate Cauchy distribution \citep[Lemma 3.7.3]{sasvari2013characteristic} \citep{devroye2014simulation}.
\item The Gaussian kernel $K(\mathbf{u})=e^{-\frac{\left\Vert \mathbf{u}\right\Vert ^{2}}{2}}$
($R=1$, $\lambda=1/2$, $\alpha=2$), also known as squared exponential
kernel \citep{rasmussen2006gaussian}. The random projection $\boldsymbol{\eta}=\boldsymbol{S}_{2}(1/2)^{\frac{1}{2}}=\boldsymbol{N}$
follows a standard multivariate Gaussian distribution.
\item The exponential power kernel $K(\mathbf{u})=e^{-\left\Vert \mathbf{u}\right\Vert ^{\alpha}}$
($R=1$, $\lambda=1$, $\alpha\in(0,2]$) \citep[(21.4)]{crooks2019field},
also known as generalized Gaussian \citep{dytso2018analytical}, generalized
normal \citep{pogany2010characteristic}, $\gamma$-exponential \citep{rasmussen2006gaussian}, stable \citep{lantuejoul2002geostatistical}, stretched exponential \citep{cardona2007history}, Kohlrausch \citep{kohlrausch1854theorie},
or Subbotin \citep{subbotin1923law} kernel. The random projection
$\boldsymbol{\eta}=\boldsymbol{S}_{\alpha}$ follows a symmetric stable
distribution \citep{devroye2006nonuniform,devroye2014simulation}.
\item The Mat\'ern-$\nu$ kernel $K(\mathbf{u})=\frac{(\sqrt{2\nu}\left\Vert \mathbf{u}\right\Vert )^{\nu}}{\Gamma(\nu)2^{\nu-1}}\mathcal{K}_{\nu}(\sqrt{2\nu}\left\Vert \mathbf{u}\right\Vert )$,
$\nu>0$ ($R=1/ G_\nu$, $\lambda=\frac{\nu}{2}$, $\alpha=2$) \citep{porcu2024matern}, also known as multivariate symmetric Laplace \citep[equation~(5.2.2)]{kotz2001laplace}, or $K$-Bessel \citep{lantuejoul2002geostatistical} kernel.
The random projection $\boldsymbol{\eta}=\boldsymbol{S}_{2}(\nu/(2G_\nu))^{\frac{1}{2}}=\boldsymbol{N}/\sqrt{G_\nu/\nu}$
follows a standard multivariate Student $t_{2\nu}$ distribution with
$2\nu$ degrees of freedom \citep[equation~(1.2)]{kotz2004multivariate}. 
\item The power kernel $K(\mathbf{u})=\frac{1}{1+\left\Vert \mathbf{u}\right\Vert ^{\alpha}}$
($R= G_1$, $\lambda=1$, $\alpha\in(0,2]$). The random projection
$\boldsymbol{\eta}=\boldsymbol{S}_{\alpha}E^{\frac{1}{\alpha}}$,
where $E$ is a standard exponential random variable, follows a Linnik
distribution \citep{linnik1953linear}, \citep{laha1961class}, \citep{devroye1990linnik},
also known as Linnik-Laha distribution.
\item The Student $t$ kernel $K(\mathbf{u})=\left(1+\frac{\left\Vert \mathbf{u}\right\Vert ^{2}}{2\beta}\right)^{-\beta}$,
$\beta>0$, with $2\beta-1$ degrees of freedom ($R=G_\beta$, $\lambda=\frac{1}{2\beta}$,
$\alpha=2$), also known as rational quadratic \citep{rasmussen2006gaussian},
or generalized inverse multiquadric \citep{hu1998collocation} kernel.
The random projection $\boldsymbol{\eta}=\boldsymbol{S}_{2}(G_\beta/(2\beta))^{\frac{1}{2}}=\boldsymbol{N}\sqrt{\frac{ G_\beta}{\beta}}$
follows a Mat\'ern distribution, also known as generalized Laplace
distribution \citep[Definition 4.1.1]{kotz2001laplace} (see \citep[Proposition 4.1.2]{kotz2001laplace}), \citep[Theorem 3.7.5]{sasvari2013characteristic},
which is a particular case of the variance-gamma distribution \citep{fischer2025variance},
also known as Bessel function distribution \citep{mckay1932bessel},
see \citep[equation (10)]{gneiting1997mixture}.
\item The generalized Cauchy kernel $K(\mathbf{u})=\frac{1}{\left(1+\left\Vert \mathbf{u}\right\Vert ^{\alpha}\right)^{\beta}}$,
$\beta>0$ ($R=G_\beta$, $\lambda=1$, $\alpha\in(0,2]$) \citep{gneiting2004stochastic},
also known as generalized Pearson~VII \citep[page 144]{crooks2019field} kernel.
The random projection $\boldsymbol{\eta}=\boldsymbol{S}_{\alpha} G_\beta^{\frac{1}{\alpha}}$
follows a generalized Linnik distribution \citep{devroye2006nonuniform,devroye2014simulation}.
The scaling $\lambda=\frac{1}{2\beta}$ proposed in Table~\ref{tab:new_random_projections}
(by analogy with Mat\'ern kernels) is such that $K(\mathbf{u})=\frac{1}{\left(1+\frac{\left\Vert \mathbf{u}\right\Vert ^{\alpha}}{2\beta}\right)^{\beta}}\underset{\beta\rightarrow\infty}{\longrightarrow}e^{-\frac{\left\Vert \mathbf{u}\right\Vert ^{\alpha}}{2}}$,
which is an exponential power kernel (with scaling $\lambda=1/2$)
which contains the Gaussian kernel as the particular case $\alpha=2$.
\item The confluent hypergeometric kernel $K(\mathbf{u})=\frac{\Gamma\left(\beta+\gamma\right)}{\Gamma\left(\gamma\right)}\,\mathcal{U}\!\left(\beta,1-\gamma,\gamma\left\Vert \mathbf{u}\right\Vert ^{2}\right)$
($R=F_{2\beta,2\gamma}$, $\lambda=\beta$, $\alpha=2$) \citep{ma2023beyond}
is a special case of the proposed Tricomi kernel with $\alpha=2$.
The random projection $\boldsymbol{\eta}=\boldsymbol{S}_{2}(\beta F_{2\beta,2\gamma})^{\frac{1}{2}}=\boldsymbol{N}\sqrt{2\gamma}\sqrt{\frac{G_{\beta}}{G_{\gamma}}}$
is a Gaussian mixture with a beta prime random radius \citep{bai2021beta, bevilacqua2025fast}, with density given by Corollary~\ref{cor:confluent_hypergeometric_spectral_density}.
\item Since $F_{2\beta,2\gamma}\overset{d}{=}(G_{\beta}/\beta)/(G_{\gamma}/\gamma)$
and $\underset{a\rightarrow\infty}{\lim}\frac{G_{a}}{a}=1$,
the proposed Tricomi kernel $K(\mathbf{u})=\frac{\Gamma\left(\beta+\gamma\right)}{\Gamma\left(\gamma\right)}\,\mathcal{U}\!\left(\beta,1-\gamma,\frac{\gamma}{\beta}\left\Vert \mathbf{u}\right\Vert ^{\alpha}\right)$
contains generalized Mat\'ern kernels (when $\beta\rightarrow\infty$)
and generalized Cauchy kernels (when $\gamma\rightarrow\infty$) as
limit cases, and therefore contains all the classical stationary kernels
(Laplace, Gaussian, Mat\'ern, Student, Power, Exponential Power,
etc.) as particular limit cases. 
\item The Fox $H$-kernels have been obtained by computing analytically
the characteristic functions of the Stacy random variable $G_{\beta}^{\frac{1}{\ell}}$
and the generalized beta prime random variable $(G_{\beta}/G_{\gamma})^{\frac{1}{\ell}}$
(Appendix~\ref{sec:characteristic_functions}). In fact, the Fox
$H$-function \citep{fox1961hfunctions} is a very general special function which encompasses
all the kernels mentioned in this research article, as well as their
spectral densities (Appendix~\ref{sec:spectral_densities}). 
\end{itemize}
\begin{rem}
Symmetric stable random vectors (Definition \ref{def:symmetric_stable})
only exist when $\alpha\in(0,2]$. As a result, for every kernel in
Table~\ref{tab:new_random_projections}, the parameter $\alpha$ is
restricted to the interval $(0,2]$ to enforce positive definiteness.
When setting $\alpha$ to a value larger than $2$, the inverse Fourier
transform \eqref{eq:f_wrt_K} of the kernel is not
a density anymore as it can take negative values. Explicit examples
from the literature include the exponential power kernel with $\alpha=4$
and $\alpha=6$ \citep{jones1992asymptotic}, and $\alpha=3$ \citep{dytso2018analytical},
as well as the generalized Cauchy kernel with $\alpha=4$ (Laha distribution \citep{laha1958example},
its inverse Fourier transform is the Silverman~kernel \citep{tsybakov2009nonparametric}) and $\alpha=6$
\citep{jones1992asymptotic}. 
\end{rem}
\begin{rem}
\label{rem:tensor_kernel}This section focuses on isotropic kernels,
which can be written as $K(\mathbf{u})=k(\left\Vert \mathbf{u}\right\Vert )$
where $k$ is a univariate kernel and $\left\Vert \mathbf{u}\right\Vert $
is the Euclidean norm of the vector $\mathbf{u}\in\mathbb{R}^{d}$.
Another classical way to construct multivariate kernels is the tensor
structure, a.k.a. separable structure \citep{vanmarcke2010random},
where $K(\mathbf{u})=\prod_{\ell=1}^{d}k(u_{\ell})$ is the product
of univariate kernels, as shown with the examples \eqref{eq:laplace_kernel}
and \eqref{eq:cauchy_kernel} from \citet{rahimi2007random}. Other
constructions include kernel sums $K(\mathbf{u})=\sum_{\ell=1}^{d}K_{\ell}(\mathbf{u})$
\citep{genton2001kernels,rasmussen2006gaussian} and kernel tensor
sums, a.k.a. additive kernels, $K(\mathbf{u})=\sum_{\ell=1}^{d}k(u_{\ell})$
\citep{rasmussen2006gaussian,durrande2012additive,langrene2019fast},
which are easier to handle due to additivity. The random projection
$\bm{\eta}=(\eta_{1},\ldots,\eta_{d})$ of a tensor kernel $K(\mathbf{u})=\prod_{\ell=1}^{d}k(u_{\ell})$
is a vector with i.i.d. components with distribution equal to the
spectral distribution of the univariate kernel $k$. In view of this,
the results from this section, in particular Table~\ref{tab:new_random_projections},
can also be used to generate new tensor kernels, by setting $d=1$
in equation~\eqref{eq:isotropic_random_projection} and simulating
the resulting univariate random projection $\eta=(\lambda R)^{\frac{1}{\alpha}}S_{\alpha}$
$d$ times independently. That being said, the results from this paper
suggest that there is little reason to favour tensor kernels over
isotropic kernels when resorting to the RFF approximation,
if only because simulating the spectral distribution of isotropic
kernels is faster than for tensor kernels, as shown in Remark \ref{rem:isotropic_vs_tensor}
in the next section.
\end{rem}
\begin{rem}
Theorem~\ref{thm:isotropic_random_projection} and Corollary~\ref{cor:isotropic_random_projection}
suggest that, in order to compare two positive definite isotropic
kernels $K_{1}$ and $K_{2}$ in $\Phi_{\infty}$, it is sufficient
to compare the distributions of their respective spectral scaling
$R_{1}\geq0$ and $R_{2}\geq0$. This is a much simpler task, since
kernels are defined in $\mathbb{R}^{d}$ while the spectral scaling
radii are univariate. As an illustration, the result in \citep{faouzi2024convergence}
connecting Cauchy and Mat\'ern covariance functions can be proved
alternatively without working directly with the density of the generalized
Linnik distribution (Proposition~\ref{prop:fourier_generalized_cauchy}),
by using the fact that the spectral scaling of generalized Cauchy
kernels has a Gamma distribution, and the fact that $G_{\beta}/\beta\overset{a.s.}{\rightarrow}1$
when $\beta\rightarrow\infty$ for a Gamma random variable $G_{\beta}$.
This alternative approach also implies that no other connection between
Cauchy and Mat\'ern covariance functions can exist besides the limiting
case identified in \citep{faouzi2024convergence}, since there is
no way to transform a Gamma distribution into an inverse Gamma distribution
by simple parameter rescaling.
\end{rem}

\section{Sampling spectral distributions\label{sec:sampling}}

In order to implement the random Fourier features approximation formula
\eqref{eq:random_fourier_features}, recalled below,
\[
K_{M}(\mathbf{u})=\frac{K(\mathbf{0})}{M}\sum_{m=1}^{M}\cos(\bm{\eta}_{m}^{\top}\mathbf{u}),\ \ \mathbf{u}\in\mathbb{R}^{d},
\]
one needs to sample from the random projection vector $\boldsymbol{\eta}$.
This article focuses on random projections of the form $\boldsymbol{\eta}_{\alpha}=R^{\frac{1}{\alpha}}\boldsymbol{S}_{\alpha}$,
$\alpha\in(0,2]$, where $\boldsymbol{S}_{\alpha}$ is a symmetric
stable random vector independent of the nonnegative random variable
$R$ (Theorem~\ref{thm:isotropic_random_projection}). Therefore,
in order to simulate the random projections from Table~\ref{tab:new_random_projections},
one needs to be able to simulate both $\boldsymbol{S}_{\alpha}$ and
$R$ in practice.

\paragraph*{Gamma random variables}

All the nonnegative random variables $R$ in Table~\ref{tab:new_random_projections}
can be obtained from independent simulations of Gamma random variables, which are elementary distributions for which simulation routines are widely available. 
Popular approaches to simulate Gamma random variables include acceptance-rejection
and numerical inversion, see \citet{luengo2022gamma}.

\paragraph*{Symmetric stable random vectors}

The following Proposition describes how to obtain simulations of symmetric
stable random vectors, in the form of scale mixtures of multivariate
Gaussian random vectors.
\begin{prop}
\label{prop:multi_exp_power}Let $\boldsymbol{N}$ be a $d$-dimensional
standard Gaussian vector, let $U_{1}$ and $U_{2}$ be two independent
standard uniform random variables, independent of $\boldsymbol{N}$,
let $W=-\log(U_{1})$ be a standard exponential random variable, and
let $\Theta=\pi\!\left(U_{2}-\frac{1}{2}\right)$ be a uniform random
variable in $\left[-\frac{\pi}{2},\frac{\pi}{2}\right]$. Then, for
any $\alpha\in(0,2]$, the multivariate symmetric stable distribution
$\boldsymbol{S}_{\alpha}$ (Definition \ref{def:symmetric_stable})
admits the following decomposition
\begin{equation}
\boldsymbol{S}_{\alpha}\overset{d}{=}\sqrt{2A_{\alpha}}\boldsymbol{N}\label{eq:multi_symmetric_stable_1}
\end{equation}
where 
\begin{equation}
A_{\alpha}:=\frac{\sin\!\left(\frac{\alpha\pi}{4}+\frac{\alpha}{2}\Theta\right)}{(\cos(\Theta))^{2/\alpha}}\left(\frac{\cos\!\left(\frac{\alpha\pi}{4}+\left(\frac{\alpha}{2}-1\right)\Theta\right)}{W}\right)^{\frac{2}{\alpha}-1}\label{eq:multi_symmetric_stable_2}
\end{equation}
\end{prop}
\begin{proof}
To obtain the Gaussian mixture representation \eqref{eq:multi_symmetric_stable_1}-\eqref{eq:multi_symmetric_stable_2}
of symmetric stable vectors, apply \citep[Proposition 2.5.2]{samorodnitsky1994stable},
use the fact that $\boldsymbol{S}_{2}\sim\mathcal{N}(\mathbf{0},2\mathbf{I}_{d})$
is a multivariate Gaussian random vector with independent components
with mean zero and variance 2 (where $\mathbf{I}_{d}$ is the $d$-dimensional
identity matrix), and use the simulation formula for nonsymmetric
stable random variable from \citep[Theorem 1.3]{nolan2020stable}
with maximum skew $\beta=1$. 
\end{proof}
\begin{cor}
\label{cor:scale_mixture_gaussian}A consequence of Theorem~\ref{thm:isotropic_random_projection}
($\boldsymbol{\eta}_{\alpha}=R^{\frac{1}{\alpha}}\boldsymbol{S}_{\alpha}$)
and Proposition~\ref{prop:multi_exp_power} ($\boldsymbol{S}_{\alpha}=\sqrt{2A_{\alpha}}\boldsymbol{N}$)
is that any continuous, positive definite, isotropic kernel $K(\mathbf{u})=k(\left\Vert \mathbf{u}\right\Vert ^{\alpha})$ in $\Phi_\infty$,
$\mathbf{u}\in\mathbb{R}^{d}$, where $\alpha\in(0,2]$, admits a
representation of its random projections as a scale mixture of Gaussians,
explicitly given by
\begin{equation}
\boldsymbol{\eta}_{\alpha}=(R^{\frac{1}{\alpha}}\sqrt{2A_{\alpha}})\boldsymbol{N},\label{eq:scale_mixture_gaussian}
\end{equation}
where the distribution of the nonnegative random variable $R\geq0$
is given by the inverse Laplace transform of $k/k(0)$,  $A_{\alpha}$
is defined by equation \eqref{eq:multi_symmetric_stable_2}, and $R$,
$A_{\alpha}$ and $\boldsymbol{N}$ are independent.
\end{cor}
This corollary has useful practical implications. For example,
it suggests that a task such as kernel learning via learning spectral
distributions can be brought down to learning the parameter $\alpha\in(0,2]$ and the distribution of
the univariate nonnegative random radius $R$. It can also suggest
variance reduction techniques by splitting the effort between the
Gaussian vector $\boldsymbol{N}$ \citep{pages2003optimal} and the
random scaling factor $R^{\frac{1}{\alpha}}\sqrt{2A_{\alpha}}$.

\begin{rem}
\label{rem:isotropic_vs_tensor}In view of equation \eqref{eq:scale_mixture_gaussian},
simulating the random projection vector of an isotropic kernel in $\Phi_\infty$ requires
to simulate $d+2$ independent random variables (the vector $\boldsymbol{N}$
of size $d$, the random variable $A_{\alpha}$ and the random variable
$R$). By contrast, simulating the random projection vector of a tensor
kernel requires to simulate $3d$ independent random variables ($d$
independent simulations of the scalar random variable $\eta=(R^{\frac{1}{\alpha}}\sqrt{2A_{\alpha}})N$
which involves the three random variables $N$, $A_{\alpha}$ and
$R$). This shows that the isotropic formulation of multivariate kernels
is almost three times more efficient than the tensor formulation when
using the random Fourier features approach with the scale mixture
representation formula \eqref{eq:scale_mixture_gaussian}.
\end{rem}

\section{Numerical experiments\label{sec:numerical}}

This section provides several numerical examples of isotropic kernels
from Table~\ref{tab:new_random_projections} along with their random
Fourier features approximation~\eqref{eq:random_fourier_features},
using the scale mixture representation of random projections~\eqref{eq:isotropic_random_projection}
established previously, and the simulation algorithms described in
the previous section. The analytical kernels involving special functions
(including the generalized Mat\'ern, Tricomi, and Fox $H$ kernels)
were evaluated numerically using methods from the Python packages
\code{scipy.special} and \code{mpmath}. Remark that a Fox $H$-function
can always be expressed in terms of a Meijer $G$-function (\code{meijerg}
function in \code{mpmath}) whenever all its parameters are rational
numbers \citep[equation~8.3.2-22]{prudnikov1990integrals}.

In the literature, numerical experiments involving random Fourier
features typically set the number of spectral simulations $M$ between
about a hundred to a few thousands \citep{rahimi2007random,avron2016quasi,li2021towards,chen2022kernel},
depending on the application, with some articles successfully using
as few as $M=20$ spectral simulations \citep{lazarogredilla2010sparse,rudi2017generalization,delbridge2020randomly}.
In this section, we deliberately set $M$ to the large value $M=d\times1000$
with $d\in\{1,2\}$ for visualization purposes, so as to greatly reduce
the inherent variability of RFF simulations, all the while keeping
the sampling error visible. 

Obviously, all the known theoretical results about the convergence
of the RFF approximation with respect to the number
of spectral points, such as \citep{rahimi2007random}, \citep{sriperumbudur2015optimal},
\citep{li2021towards} or \citep{yao2023error}, still apply to the
kernels discussed in this paper. This means that one can keep increasing
$M$ further to make the sampling error invisible to the naked eye.
In practice though, much smaller values of $M$ are likely to be sufficient,
as the end goal of kernel approximations in machine learning is not
to get a highly precise approximation of the kernel function in the
whole domain, but to obtain stable classification or regression predictions,
for a given dataset, from the kernel machine under consideration.

The two key takeaways of these simple numerical experiments are the following:
\begin{itemize}
\item When sampling directly from the spectral density is possible (such
as in the case of Gaussian or Mat\'ern kernels), our proposed spectral
scale mixture sampling~\eqref{eq:isotropic_random_projection} does
not make any difference to the RFF methodology. However, when the
spectral density is intractable, such as in the case of the spectral
densities reported in Appendix~\ref{sec:spectral_densities}, which
can only be expressed in terms of $H$-functions for general parameters,
our spectral scale mixture sampling formula makes the implementation
of RFF possible and straightforward. In other words, our work greatly
\emph{facilitates the implementation} of RFF for a very broad class
of kernel functions in $\Phi_{\infty}$.
\item We observe numerically, as expected, that the RFF approximations of
the kernels implemented below converge uniformly to the true kernel
functions when increasing the number of spectral points $M$. This
serves as a useful check of the analytical kernels formulas established
in Table~\ref{tab:new_random_projections}, in particular of their
scaling constants for general parameters.
\end{itemize}

\begin{figure}[H]
\begin{centering}
\includegraphics[width=1\textwidth]{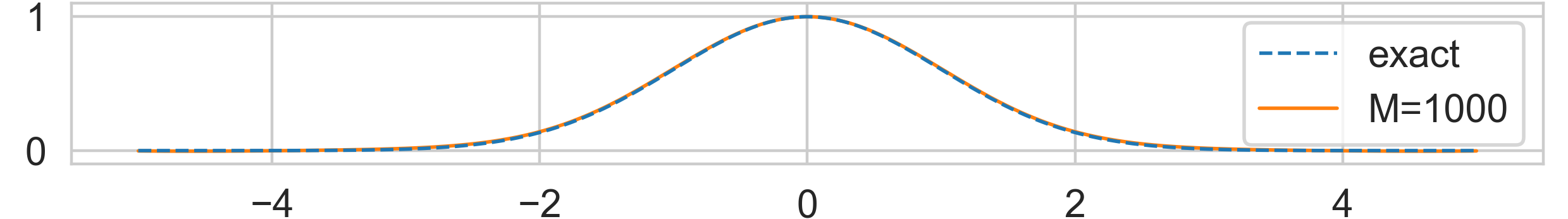}
\par\end{centering}
\caption{Univariate Gaussian kernel and its random Fourier features approximation
\eqref{eq:random_fourier_features} using $M=1000$ random projections.\label{fig:gaussian_1d}}
\end{figure}

\vspace{-4mm}
\begin{figure}[H]
\begin{minipage}[t]{0.48\columnwidth}%
\includegraphics[width=0.38\paperwidth]{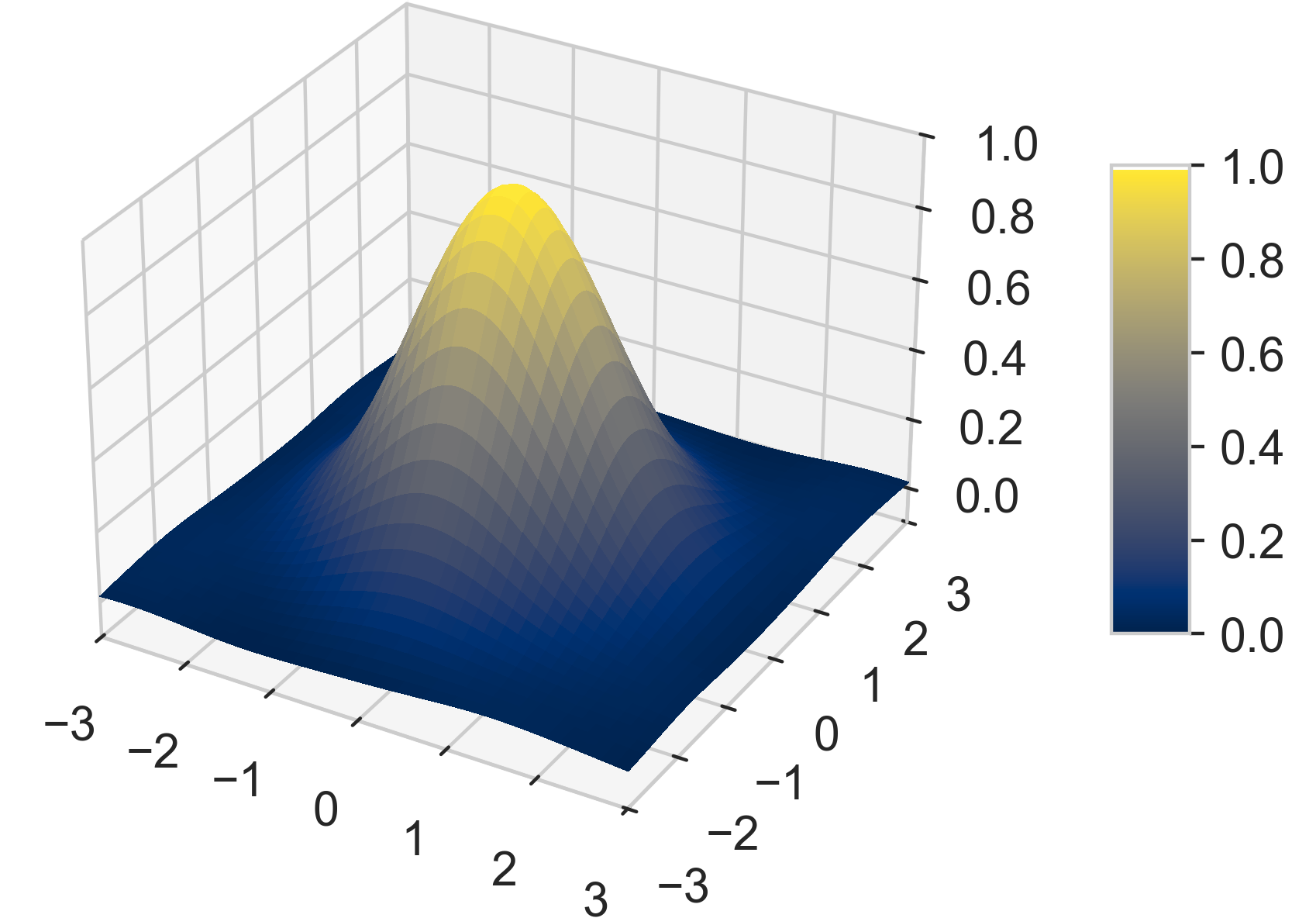}%
\end{minipage}\hfill{}%
\begin{minipage}[t]{0.48\columnwidth}%
\includegraphics[width=0.38\paperwidth]{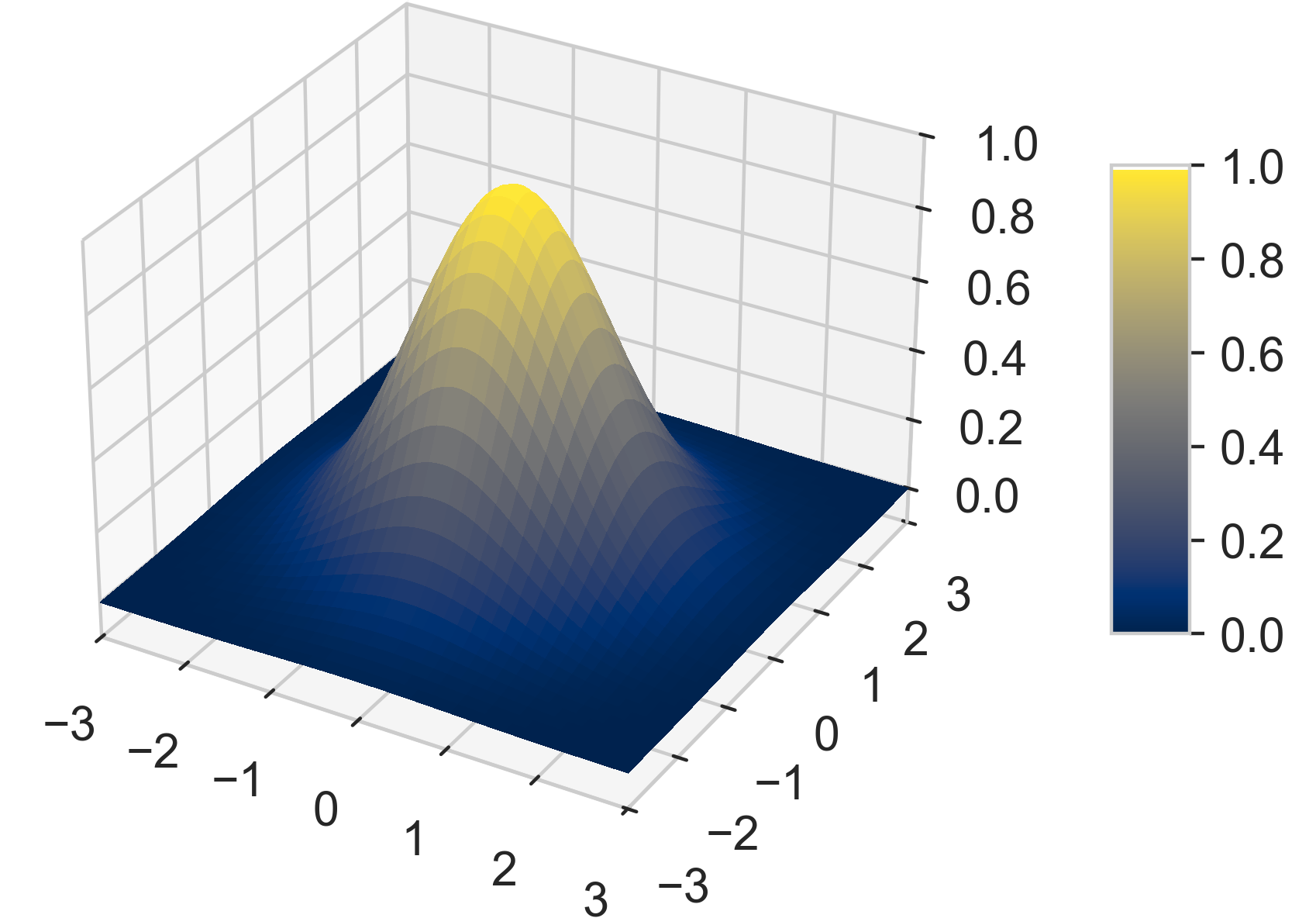}%
\end{minipage}

\caption{Bivariate Gaussian kernel and its random Fourier features approximation
\eqref{eq:random_fourier_features} (left) using $M=2000$ random
projections.\label{fig:gaussian_2d}}
\end{figure}

\begin{figure}[H]
\begin{centering}
\includegraphics[width=1\textwidth]{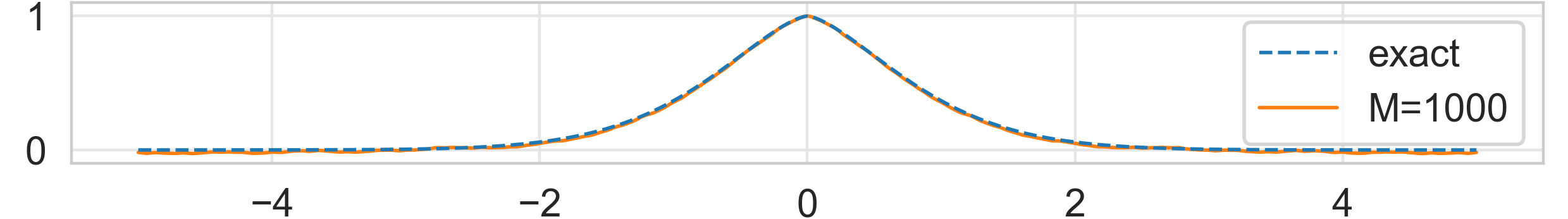}
\par\end{centering}
\caption{Univariate exponential power kernel with $\alpha=1.5$ and its random
Fourier features approximation \eqref{eq:random_fourier_features}
using $M=1000$ random projections.\label{fig:exponential_power_1d}}
\end{figure}

\vspace{-4mm}
\begin{figure}[H]
\begin{minipage}[t]{0.48\columnwidth}%
\includegraphics[width=0.38\paperwidth]{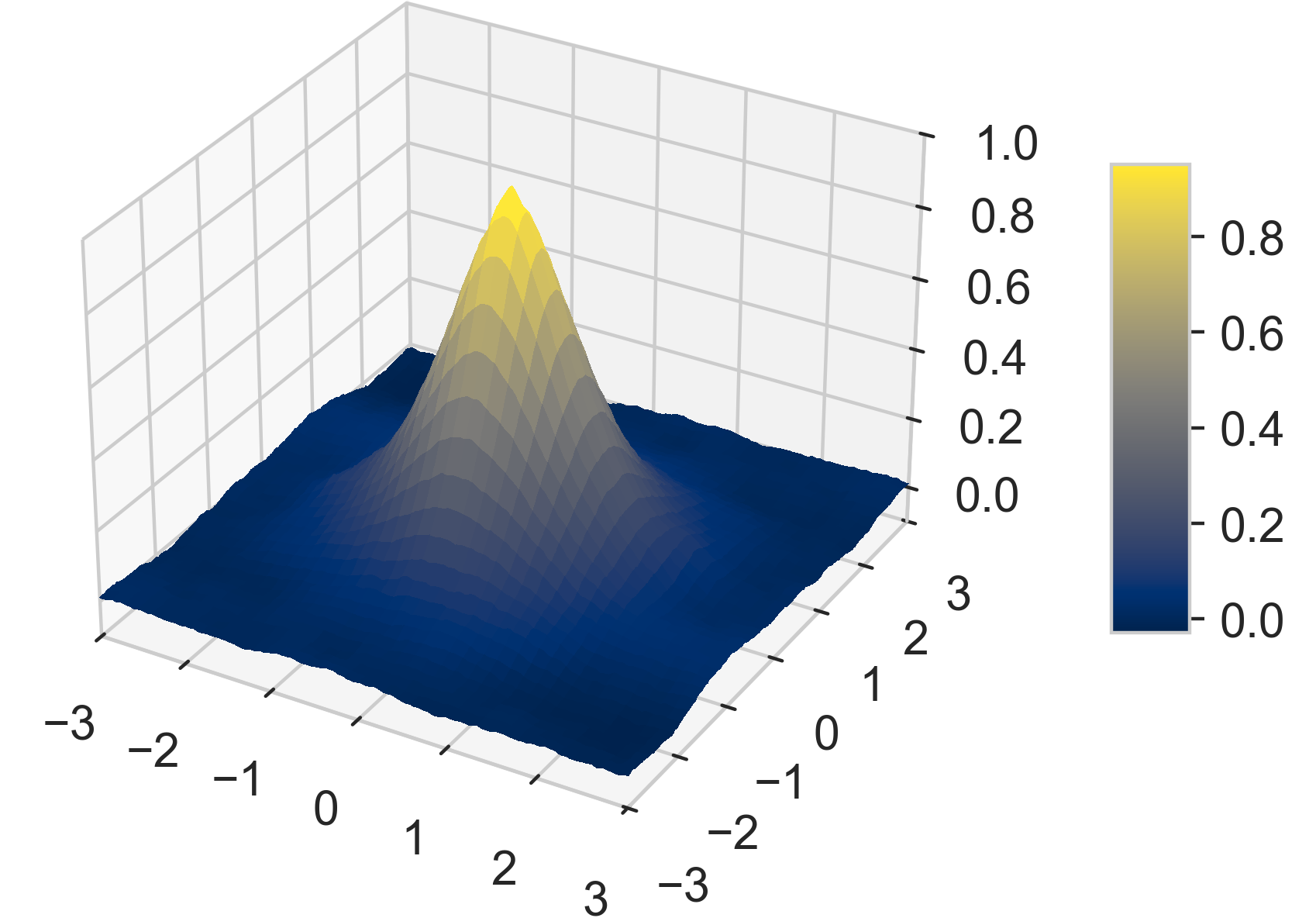}%
\end{minipage}\hfill{}%
\begin{minipage}[t]{0.48\columnwidth}%
\includegraphics[width=0.38\paperwidth]{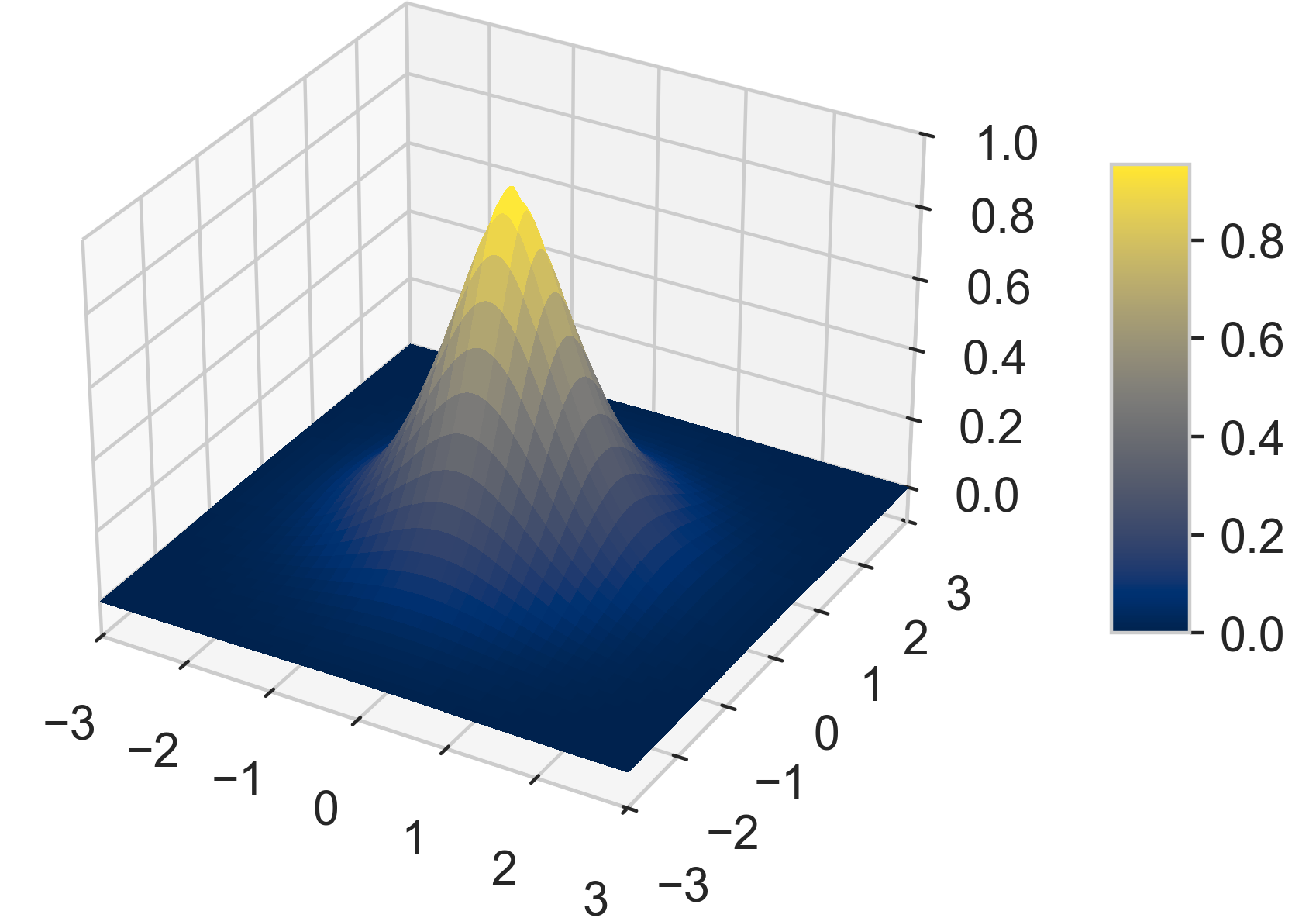}%
\end{minipage}

\caption{Bivariate exponential power kernel with $\alpha=1.5$ (right) and
its random Fourier features approximation \eqref{eq:random_fourier_features}
(left) using $M=2000$ random projections.\label{fig:exponential_power_2d}}
\end{figure}
\begin{figure}[H]
\begin{centering}
\includegraphics[width=1\textwidth]{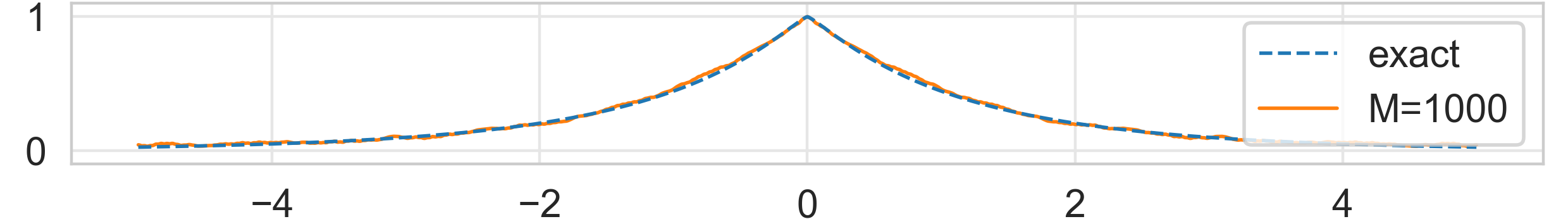}
\par\end{centering}
\caption{Univariate generalized Mat\'ern kernel with $(\alpha,\beta)=(1.5,1.0)$
and its random Fourier features approximation \eqref{eq:random_fourier_features}
using $M=1000$ random projections.\label{fig:generalized_matern_1d}}
\end{figure}

\vspace{-4mm}
\begin{figure}[H]
\begin{minipage}[t]{0.48\columnwidth}%
\includegraphics[width=0.38\paperwidth]{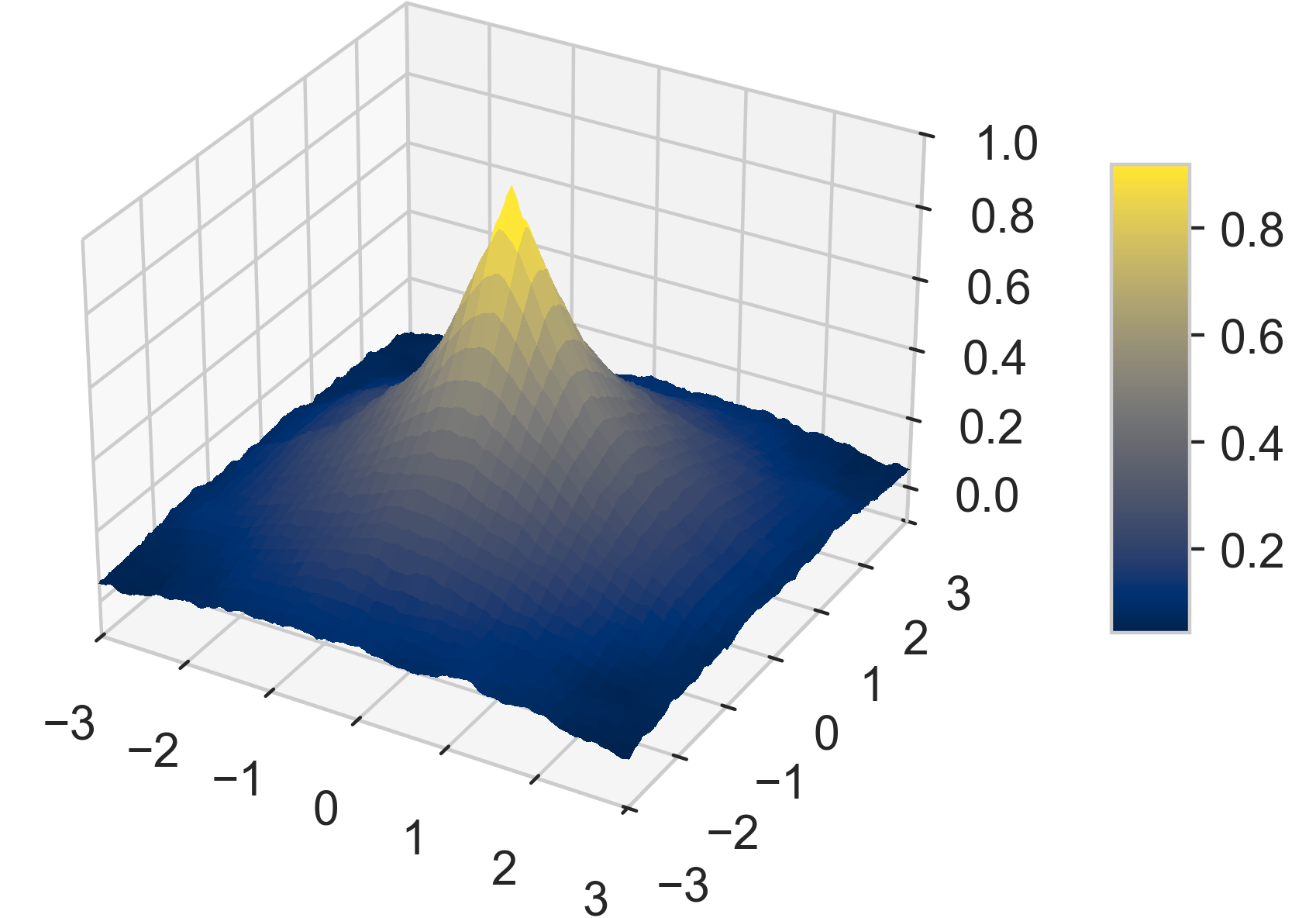}%
\end{minipage}\hfill{}%
\begin{minipage}[t]{0.48\columnwidth}%
\includegraphics[width=0.38\paperwidth]{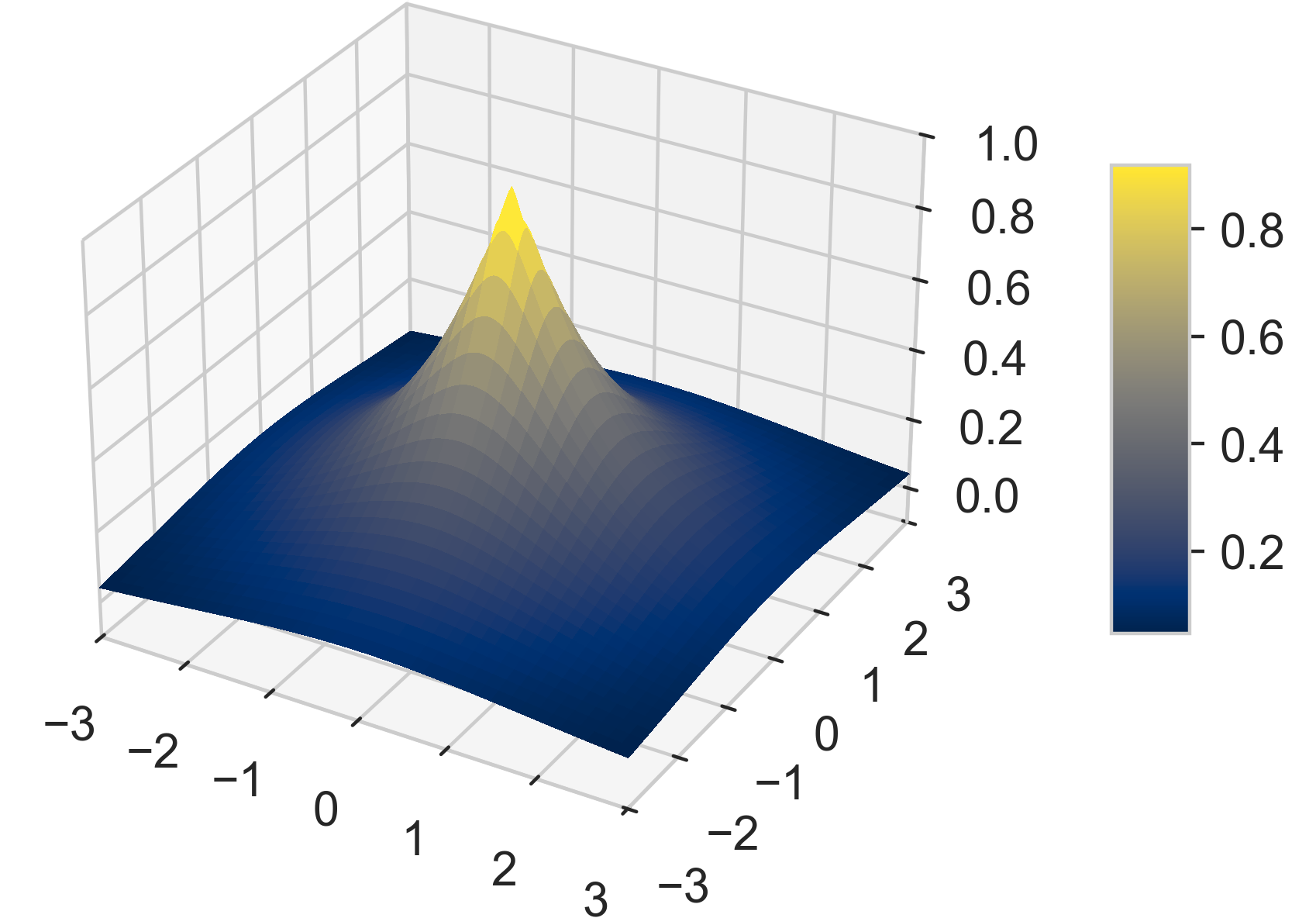}%
\end{minipage}

\caption{Bivariate generalized Mat\'ern kernel with $(\alpha,\beta)=(1.5,1.0)$
(right) and its random Fourier features approximation \eqref{eq:random_fourier_features}
(left) using $M=2000$ random projections.\label{fig:matern_32_2d}}
\end{figure}

\begin{figure}[H]
\begin{centering}
\includegraphics[width=1\textwidth]{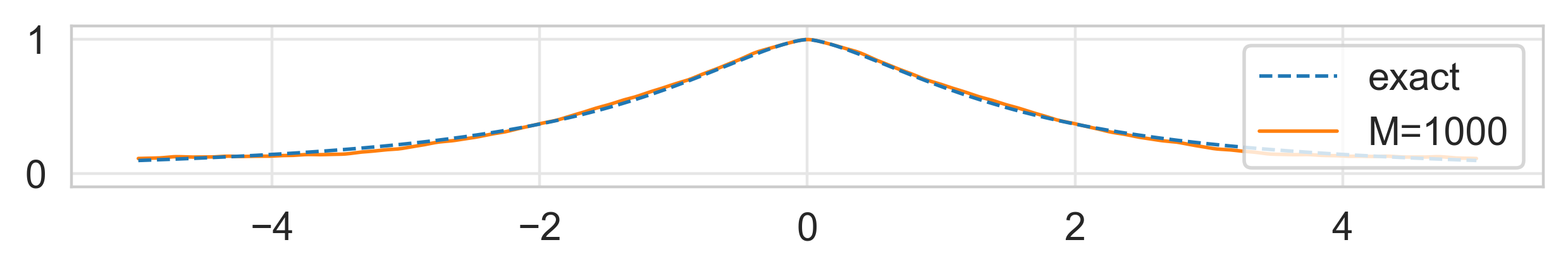}
\par\end{centering}
\caption{Univariate generalized Cauchy kernel with $(\alpha,\beta)=(1.5,1.5)$
and its random Fourier features approximation \eqref{eq:random_fourier_features}
using $M=1000$ random projections.\label{fig:generalized_cauchy_1d}}
\end{figure}

\vspace{-4mm}
\begin{figure}[H]
\begin{minipage}[t]{0.48\columnwidth}%
\includegraphics[width=0.38\paperwidth]{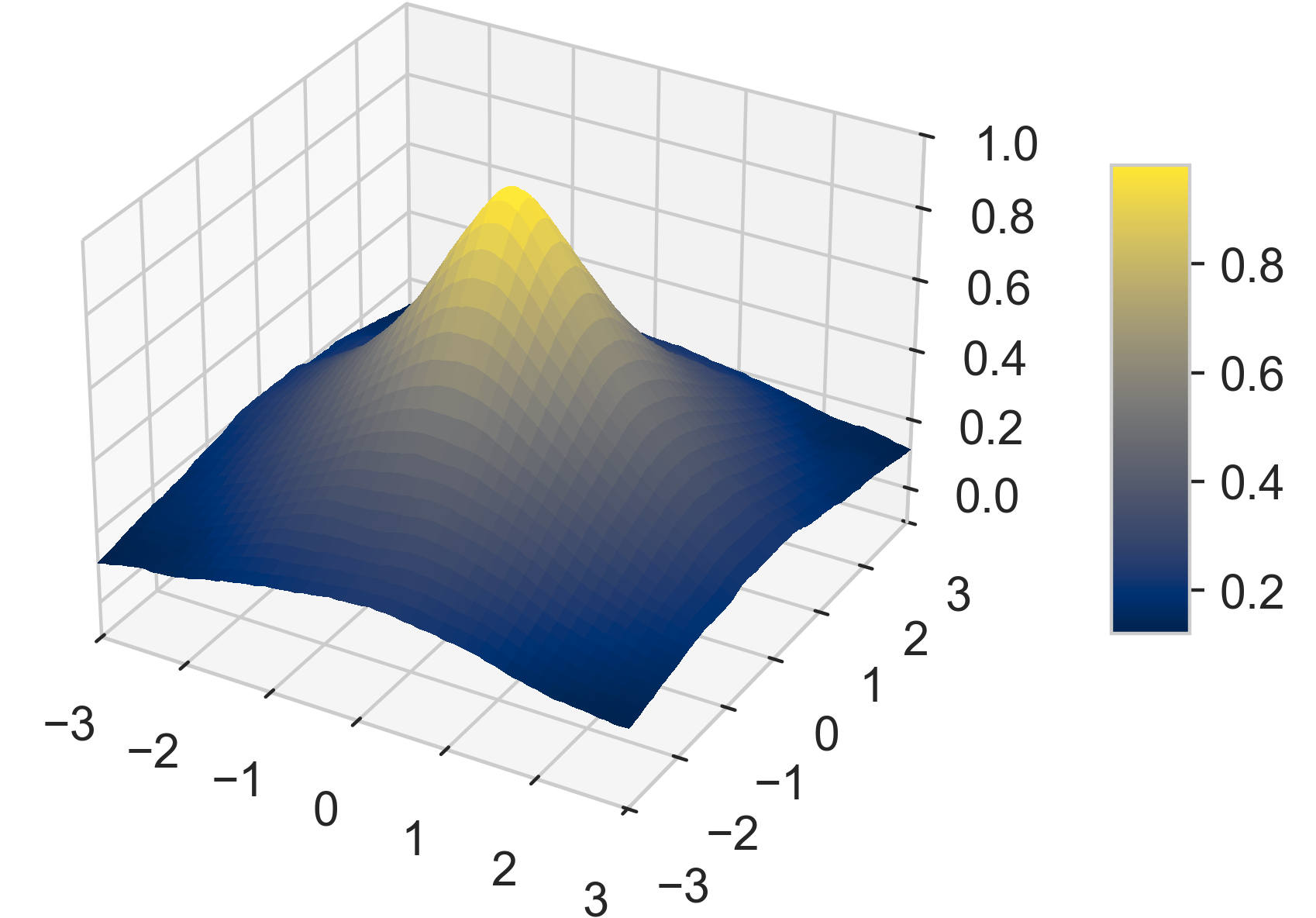}%
\end{minipage}\hfill{}%
\begin{minipage}[t]{0.48\columnwidth}%
\includegraphics[width=0.38\paperwidth]{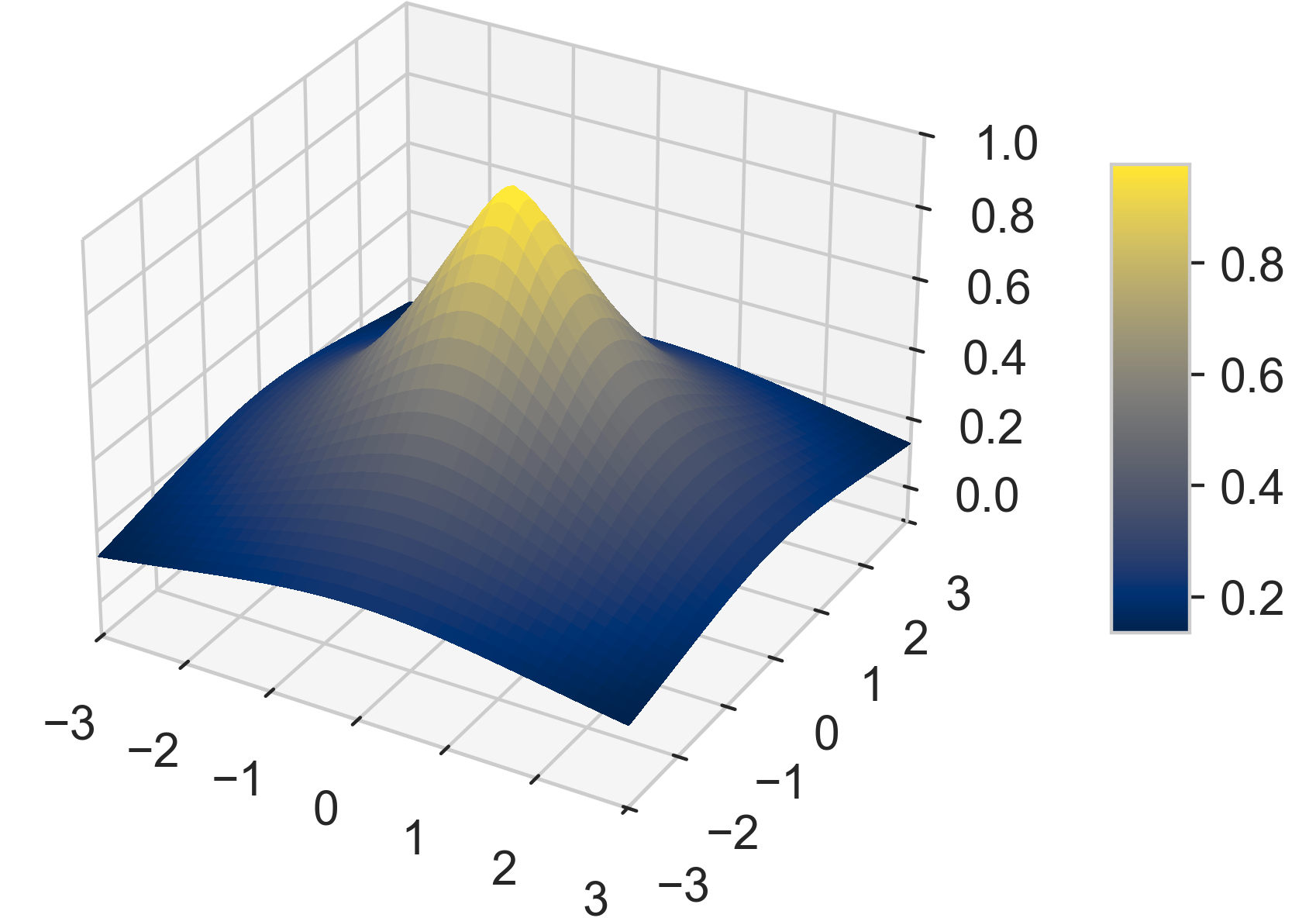}%
\end{minipage}

\caption{Bivariate generalized Cauchy kernel with $(\alpha,\beta)=(1.5,1.5)$
(right) and its random Fourier features approximation \eqref{eq:random_fourier_features}
(left) using $M=2000$ random projections. \label{fig:generalized_cauchy_2d}}
\end{figure}

\begin{figure}[H]
\begin{centering}
\includegraphics[width=1\textwidth]{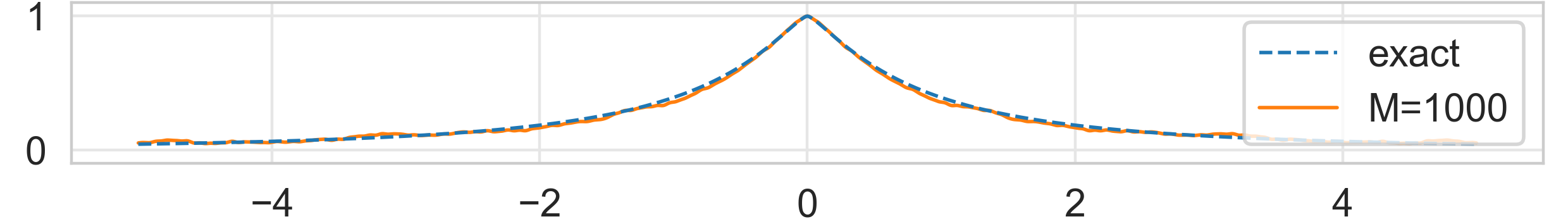}
\par\end{centering}
\caption{Univariate Tricomi kernel with $(\alpha,\beta,\gamma)=(1.5,1.5,1.5)$
and its random Fourier features approximation \eqref{eq:random_fourier_features}
using $M=1000$ random projections.\label{fig:tricomi_1d}}
\end{figure}

\vspace{-4mm}
\begin{figure}[H]
\begin{minipage}[t]{0.48\columnwidth}%
\includegraphics[width=0.38\paperwidth]{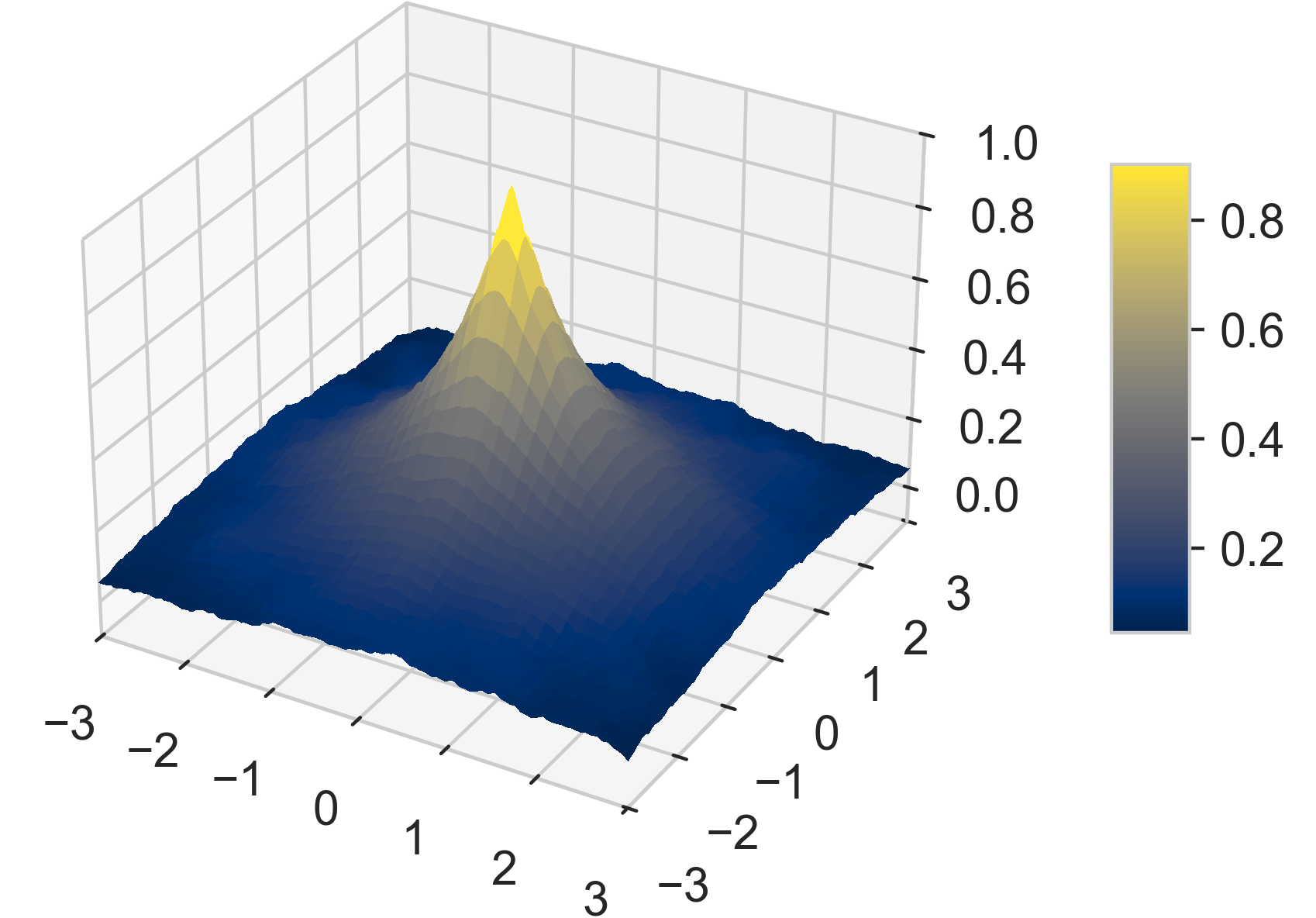}%
\end{minipage}\hfill{}%
\begin{minipage}[t]{0.48\columnwidth}%
\includegraphics[width=0.38\paperwidth]{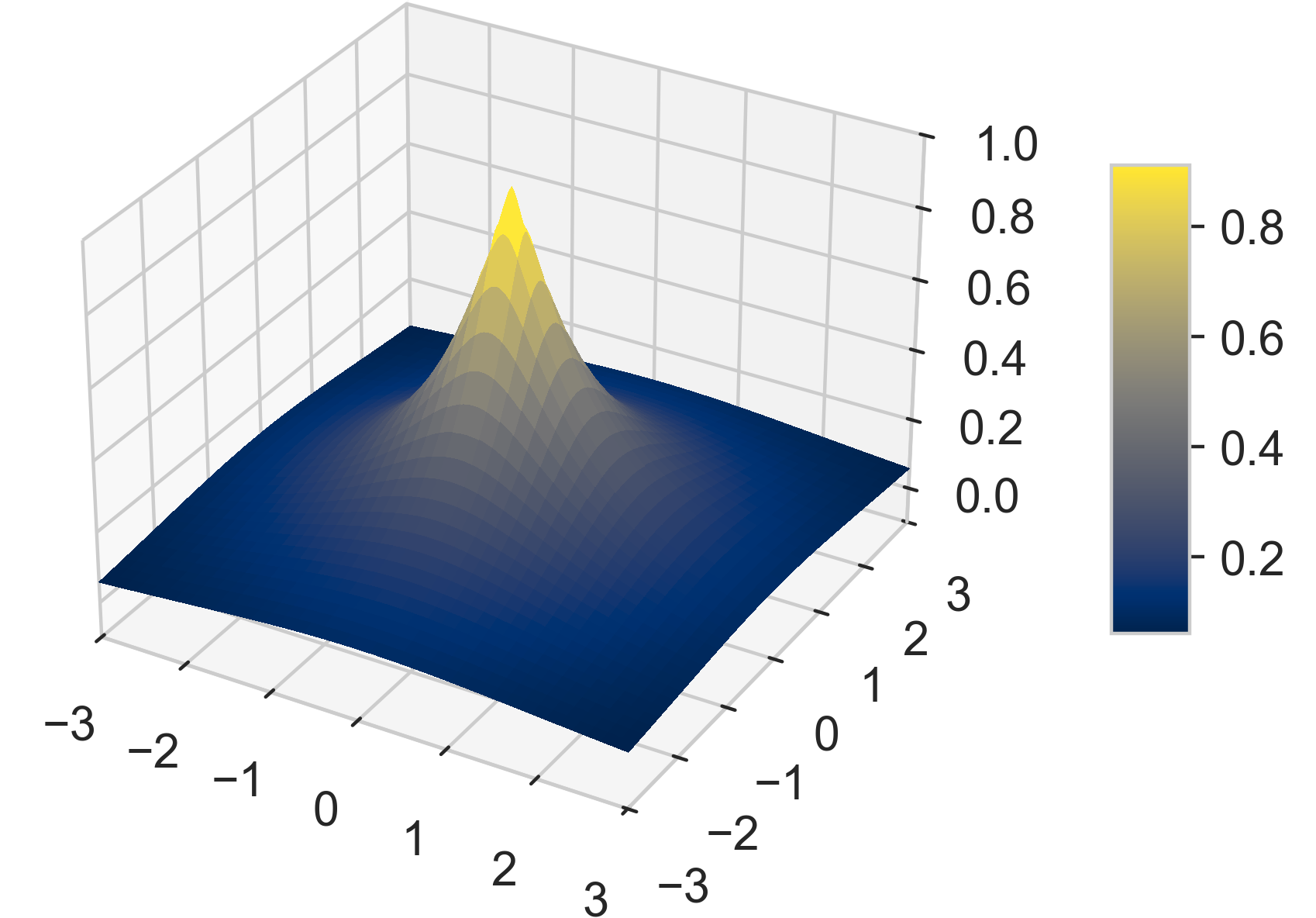}%
\end{minipage}

\caption{Bivariate Tricomi kernel with $(\alpha,\beta,\gamma)=(1.5,1.5,1.5)$
(right) and its random Fourier features approximation \eqref{eq:random_fourier_features}
(left) using $M=2000$ random projections. \label{fig:tricomi_2d}}
\end{figure}
\begin{figure}[H]
\begin{centering}
\includegraphics[width=1\textwidth]{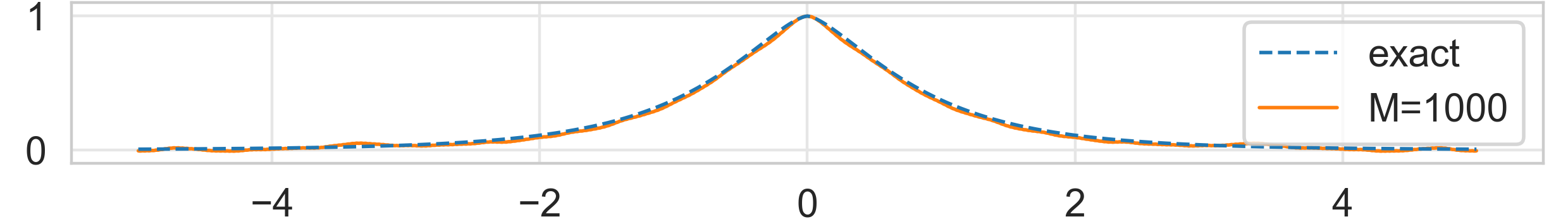}
\par\end{centering}
\caption{Univariate Fox $H_{1,2}^{2,1}$ kernel with $(\alpha,\beta,\gamma,\ell)=(1.5,1.5,1.5,2.0)$
and its random Fourier features approximation \eqref{eq:random_fourier_features}
using $M=1000$ random projections.\label{fig:foxh_1d}}
\end{figure}

\vspace{-4mm}
\begin{figure}[H]
\begin{minipage}[t]{0.48\columnwidth}%
\includegraphics[width=0.38\paperwidth]{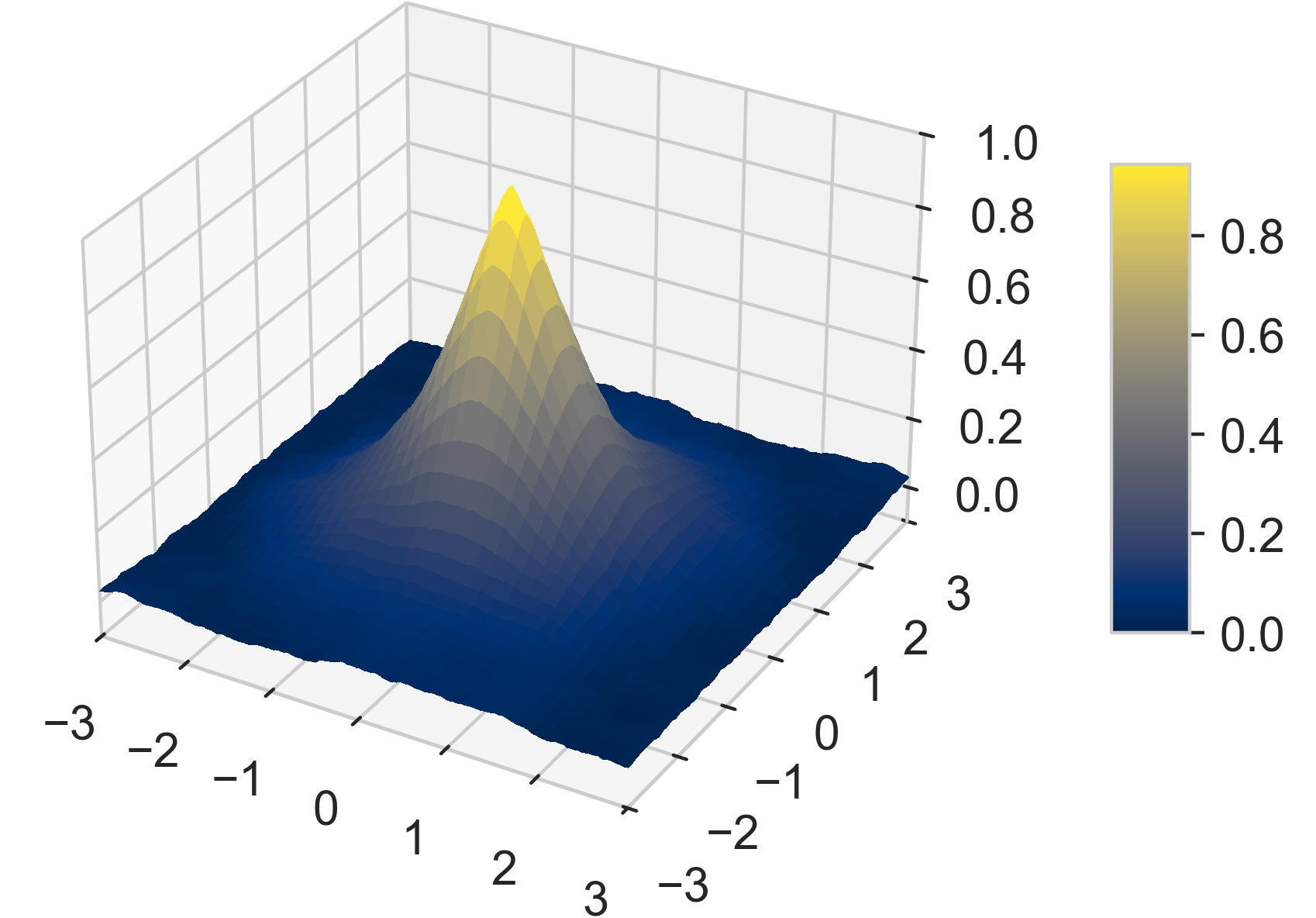}%
\end{minipage}\hfill{}%
\begin{minipage}[t]{0.48\columnwidth}%
\includegraphics[width=0.38\paperwidth]{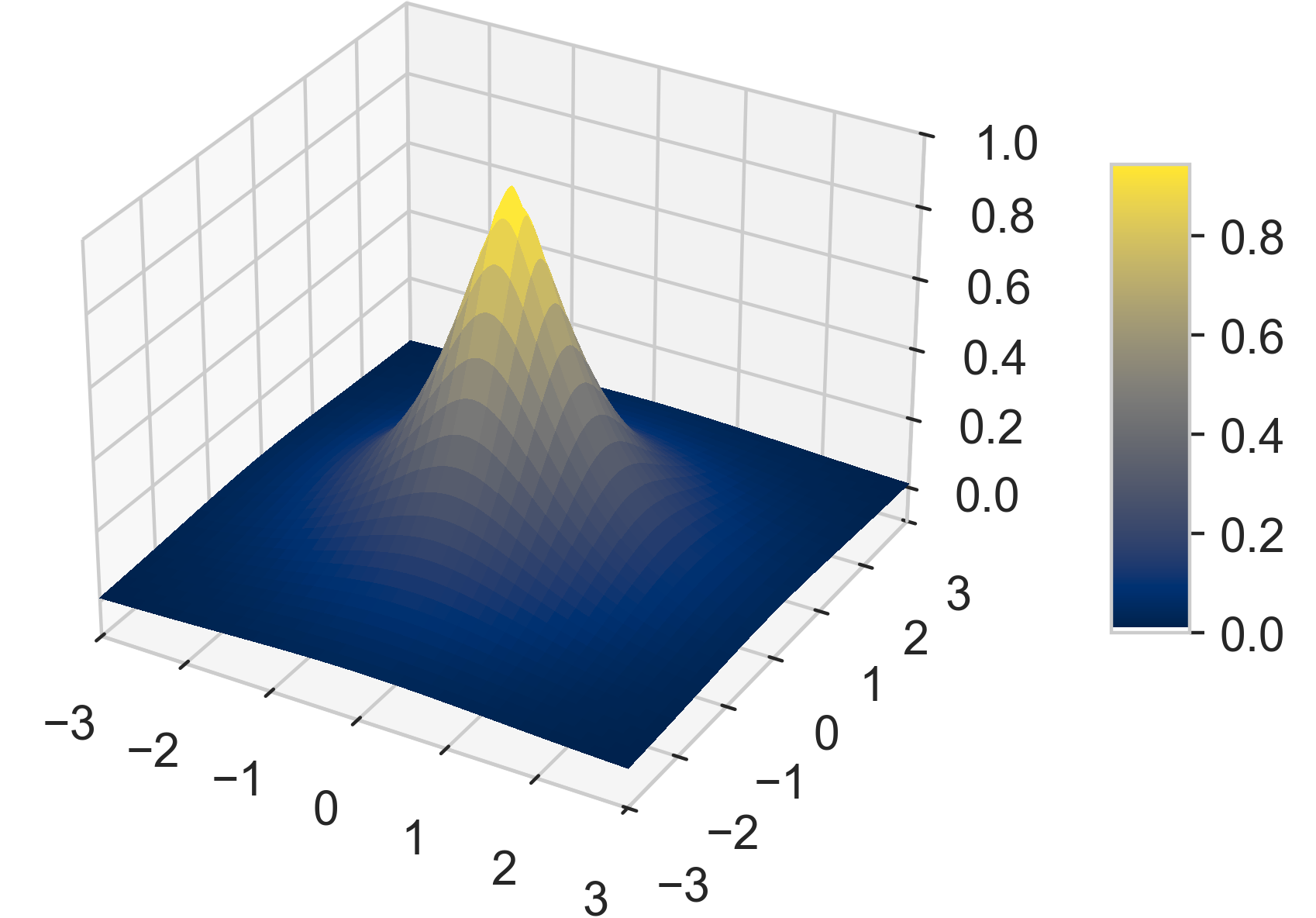}%
\end{minipage}

\caption{Bivariate Fox $H_{1,2}^{2,1}$ kernel with $(\alpha,\beta,\gamma,\ell)=(1.5,1.5,1.5,2.0)$
(right) and its random Fourier features approximation \eqref{eq:random_fourier_features}
(left) using $M=2000$ random projections. \label{fig:foxh_2d}}
\end{figure}

\section{Conclusion\label{sec:conclusion}}

In this paper, we proved that if an isotropic kernel $K(\mathbf{u})=k(\left\Vert \mathbf{u}\right\Vert ^{2})$,
$\mathbf{u}\in\mathbb{R}^{d}$ belongs to the set $\Phi_{\infty}$
of positive definite kernels in $\mathbb{R}^{d}$ for all $d\geq1$,
then the same is true for all the kernels $K_{\alpha}(\mathbf{u})=k(\left\Vert \mathbf{u}\right\Vert ^{\alpha})$,
$\mathbf{u}\in\mathbb{R}^{d}$, $\alpha\in(0,2]$. This result can
be used to generalize and create new kernels in $\Phi_{\infty}$.
Examples include the generalized Mat\'ern kernel and the Tricomi
kernel, which generalize the classical Mat\'ern and confluent hypergeometric
kernels, respectively. We established analytical expressions for the
spectral densities of these generalized kernels in terms of the Fox
$H$ generalized hypergeometric function.

Moreover, we proved that the spectral distribution of these generalized
kernels $K_{\alpha}(\mathbf{u})=k(\left\Vert \mathbf{u}\right\Vert ^{\alpha})$
admits the mixture representation $\boldsymbol{\eta}_{\alpha}=R^{\frac{1}{\alpha}}\boldsymbol{S}_{\alpha}$,
where $\boldsymbol{S}_{\alpha}$ is a symmetric stable random vector,
and $R$ is a nonnegative random variable, independent of $\boldsymbol{S}_{\alpha}$,
whose distribution is defined by the inverse Laplace transform of
$k/k(0)$. This result has direct applications for the random Fourier
features methodology in machine learning, which is a kernel decomposition
technique based on the Monte Carlo simulation of spectral densities.
Indeed, while sampling directly from a Fox $H$ density would be a
great challenge in practice, simulating the scale mixture $R^{\frac{1}{\alpha}}\boldsymbol{S}_{\alpha}$
is very simple since $\boldsymbol{S}_{\alpha}$ can be sampled using
a simple extension of the Box-Muller formula, and the random variable
$R$ is straightforward to simulate for most kernel functions of interest.
Our numerical experiments confirm the exactness of the proposed method.

We believe that further developments ought to be achievable at this
intersection between the theories of spectral distributions and stable
distributions.  One idea for future work could be to study analytically the effect
of the newly introduced parameter $\alpha$ on the convergence of
the random Fourier features approximation. Another idea could be to
aim to extend this spectral scale mixture representation to more general
classes of kernels, for example non-stationary kernels, with useful
implications for kernel-based machine learning.

\section*{Acknowledgements}

The authors would like to thank Dr. Christian Walder (Google DeepMind) for fruitful discussions,  and Prof.~Emilio Porcu (Khalifa University) for letting us know about an error in a previous version of this work.  Nicolas Langren\'e acknowledges the partial support of the Guangdong Provincial/Zhuhai Key Laboratory of IRADS (2022B1212010006) and the BNBU Start-up Research Fund UICR0700041-22. Pierre Gruet and Xavier Warin acknowledge support from the FiME Lab.

\bibliographystyle{apalike}
\bibliography{biblio}

\appendix

\section{Analytical characteristic functions\label{sec:characteristic_functions}}

This appendix provides analytical formulas for the Laplace transforms
and the characteristic functions of the densities of two particular non-negative
random variables, namely the Stacy distribution \citep[equation~(11.2)]{crooks2019field}
(Proposition~\ref{prop:laplace_stacy}), a.k.a. Amoroso distribution,
obtained as the power $G_{\beta}^{\frac{1}{\ell}}$ of a Gamma random
variable, and the generalized beta prime distribution \citep[equation~(18.1)]{crooks2019field}
(Proposition~\ref{prop:laplace_generalized_beta_prime}), obtained
as the power $(G_{\beta}/G_{\gamma})^{\frac{1}{\ell}}$ of the ratio
of two independent Gamma random variables. These formulas are expressed
in terms of the Fox $H$-function $H_{p,q}^{m,n}$ \citep{mathai2010hfunction},
and we used them to create the Fox $H$-kernels listed in Table~\ref{tab:new_random_projections}. 
\begin{prop}
\label{prop:laplace_stacy}The Laplace transform $\mathcal{L}$ and
the characteristic function $\phi$ of the Stacy random variable, whose density is defined by
\[
f(x)=\frac{\ell}{\Gamma(\beta)}x^{\ell\beta-1}e^{-x^{\ell}}\mathbbm{1}_{\{x>0\}}
\]
where $\beta>0$, $\ell>0$, are given by
\begin{align*}
\mathcal{L}(s) & =\frac{1}{\Gamma(\beta)}H_{1,1}^{1,1}\Big(s\,\Big|\begin{array}{l}
{\scriptstyle (1-\beta,\frac{1}{\ell})}\\
{\scriptstyle (0,1)}
\end{array}\Big),\\
\phi(s) & =\frac{1}{\Gamma(\beta)}H_{1,1}^{1,1}\Big(-is\,\Big|\begin{array}{l}
{\scriptstyle (1-\beta,\frac{1}{\ell})}\\
{\scriptstyle (0,1)}
\end{array}\Big).
\end{align*}
\end{prop}
\begin{proof}
Use the formula $e^{-s}=G_{0,1}^{1,0}\big(s\big|0\big)$, then \citep[equation~8.3.1-21]{prudnikov1990integrals}
and \citep[equation~2.25.1]{prudnikov1990integrals}. From \citep[Theorem~1.1 case~7]{mathai2010hfunction},
this $H$-function exists and is analytical for all $\ell>0$ and
$s\neq0$, and the case $s=0$ can be verified separately. Then, since the Stacy distribution is a nonnegative distribution,
the equality $\phi(s)=\mathcal{\mathcal{L}}(-is)$ holds for every
$s\geq0$.
\end{proof}
Proposition~\ref{prop:laplace_stacy} generalizes the Laplace transform
formula \citep[equation~(3.462-1)]{gradshteyn2014table} corresponding
to the case $\ell=2$, and simplifies the formula \citep[equation~(2.2.1-22)]{prudnikov1992integrals}
which requires $\ell$ to be a rational number.

\begin{prop}
\label{prop:laplace_generalized_beta_prime}The Laplace transform
$\mathcal{L}$ and the characteristic function $\phi$ of the generalized beta prime random variable, whose density is defined by
\[
f(x)=\ell\frac{\Gamma(\beta+\gamma)}{\Gamma(\beta)\Gamma(\gamma)}x^{\ell\beta-1}(1+x^{\ell})^{-\beta-\gamma}\mathbbm{1}_{\{x>0\}}
\]
where $\beta>0$, $\gamma>0$, $\ell>0$, are given by
\begin{align*}
\mathcal{L}(s) & =\frac{1}{\Gamma(\beta)\Gamma(\gamma)}H_{1,2}^{2,1}\Big(s\,\Big|\begin{array}{l}
{\scriptstyle (1-\beta,\frac{1}{\ell})}\\
{\scriptstyle (0,1),(\gamma,\frac{1}{\ell})}
\end{array}\Big),\\
\phi(s) & =\frac{1}{\Gamma(\beta)\Gamma(\gamma)}H_{1,2}^{2,1}\Big(-is\,\Big|\begin{array}{l}
{\scriptstyle (1-\beta,\frac{1}{\ell})}\\
{\scriptstyle (0,1),(\gamma,\frac{1}{\ell})}
\end{array}\Big).
\end{align*}
\end{prop}
\begin{proof}
Use the formulas $e^{-s}=G_{0,1}^{1,0}\big(s\big|0\big)$ and $\left(1+t\right)^{-\beta-\gamma}=\frac{1}{\Gamma(\beta+\gamma)}G_{1,1}^{1,1}\Big(t\Big|\begin{array}{l}
{\scriptstyle 1-\beta-\gamma}\\
{\scriptstyle 0}
\end{array}\Big)$, then \citep[equation~8.3.1-21]{prudnikov1990integrals} and \citep[equation~2.25.1]{prudnikov1990integrals}.
From \citep[Theorem~1.1 case~1]{mathai2010hfunction}, this $H$-function
exists and is analytical for all $\ell>0$ and $s\neq0$, and the case $s=0$ can be verified separately. Then, since
the generalized beta prime distribution is a nonnegative distribution,
the equality $\phi(s)=\mathcal{\mathcal{L}}(-is)$ holds for every
$s\geq0$.
\end{proof}
Proposition~\ref{prop:laplace_generalized_beta_prime} generalizes
the Laplace transform formula \citep[equation~(3.389-2)]{gradshteyn2014table}
corresponding to the case $\ell=2$, and simplifies the formula \citep[equation~(2.1.4-8)]{prudnikov1992integrals}
which requires $\ell$ to be a rational number.

\section{Analytical spectral densities\label{sec:spectral_densities}}

This appendix provides analytical formulas for the spectral densities
of several kernels from Table~\ref{tab:new_random_projections}. Those not listed here can be addressed in the exact same manner. These
densities are expressed in terms of the Fox $H$-function $H_{p,q}^{m,n}$
\citep{mathai2010hfunction}. Proofs in this appendix make use of
the Meijer $G$-function $G_{p,q}^{m,n}\Big(z\Big|\begin{array}{l}
{\scriptstyle a_{1},\ldots,a_{p}}\\
{\scriptstyle b_{1},\ldots,b_{q}}
\end{array}\Big)$ \citep[16.17]{dlmf}. We refer to \citep{prudnikov1990integrals}
and https://functions.wolfram.com for properties of the $G$-function
and its connections with other classical and special functions. We
make use of the following Lemma.
\begin{lem}
\label{lem:fourier_meijer}Let $K(\mathbf{u})=k(\left\Vert \mathbf{u}\right\Vert ),\ \mathbf{u}\in\mathbb{R}^{d}$
be a continuous, absolutely integrable kernel function, and let $\alpha>0$.
Then, the Fourier transform of $K$ is given by
\begin{equation}
\mathcal{F}(\mathbf{x})=f(\left\Vert \mathbf{x}\right\Vert )=\frac{1}{\alpha2^{d-1}\pi^{\frac{d}{2}}}\int_{0}^{\infty}\!s^{\frac{d}{\alpha}-1}k(s^{\frac{1}{\alpha}})G_{0,2}^{1,0}\bigg(\frac{\left\Vert \mathbf{x}\right\Vert ^{2}}{4}s^{\frac{2}{\alpha}}\bigg|0,1-\frac{d}{2}\bigg)ds\ ,\ \mathbf{x}\in\mathbb{R}^{d}\label{eq:fourier_meijer}
\end{equation}
\end{lem}
\begin{proof}
Since $K$ is continuous and absolutely integrable, its Fourier transform
is given, using our notational convention \eqref{eq:K_wrt_f}-\eqref{eq:f_wrt_K}
for Fourier transforms, by \citep[Theorem~B.1]{fasshauer2007meshfree}
\begin{align*}
f(\left\Vert \mathbf{x}\right\Vert ) & =\frac{1}{(2\pi)^{\frac{d}{2}}\left\Vert \mathbf{x}\right\Vert ^{\frac{d}{2}-1}}\int_{0}^{\infty}t^{\frac{d}{2}}k(t)J_{\frac{d}{2}-1}(\left\Vert \mathbf{x}\right\Vert t)dt\\
 & =\frac{1}{2^{d-1}\pi^{\frac{d}{2}}\Gamma(d/2)}\int_{0}^{\infty}t^{d-1}k(t){}_{0}F_{1}\bigg(;\frac{d}{2},-\frac{\left\Vert \mathbf{x}\right\Vert ^{2}}{4}t^{2}\bigg)dt\\
 & =\frac{1}{\alpha2^{d-1}\pi^{\frac{d}{2}}}\int_{0}^{\infty}s^{\frac{d}{\alpha}-1}k(s^{\frac{1}{\alpha}})G_{0,2}^{1,0}\bigg(\frac{\left\Vert \mathbf{x}\right\Vert ^{2}}{4}s^{\frac{2}{\alpha}}\bigg|0,1-\frac{d}{2}\bigg)ds
\end{align*}
where we used the special functions formulas $J_{\nu}(z)=\frac{1}{\Gamma(\nu+1)}\left(\frac{z}{2}\right)^{\nu}{}_{0}F_{1}\Big(;\nu+1,-\frac{z^{2}}{4}\Big)$
and $_{0}F_{1}(;b,z)=\Gamma(b)G_{0,2}^{1,0}\big(-z\big|0,1-b\big)$.
\end{proof}
Using Lemma~\ref{lem:fourier_meijer}, one can obtain analytical
expressions for Fourier transforms of isotropic kernels by expressing
$k$ in terms of $G$-function, converting the two $G$-functions
in \eqref{eq:fourier_meijer} to $H$-functions using \citep[equation~8.3.1-21]{prudnikov1990integrals},
and finally integrating analytically into a $H$-function using \citep[equation~2.25.1]{prudnikov1990integrals}.
This is the approach used throughout this appendix.

\subsection{Exponential power kernel\label{subsec:exponential_power_kernel}}

The density of $\alpha$-stable random variables (univariate, possibly
non-symmetric) can be expressed in terms of the Fox $H_{2,2}^{1,1}$-function
\citep{schneider1986stable,rathie2016exact}. The following
proposition provides a formula for symmetric, multivariate isotropic
$\alpha$-stable densities.
\begin{prop}
\label{prop:fourier_exponential_power}The Fourier transform of the
exponential power kernel 
\begin{equation}
K(\mathbf{u})=k(\left\Vert \mathbf{u}\right\Vert )=e^{-\left\Vert \mathbf{u}\right\Vert ^{\alpha}}\ ,\ \mathbf{u}\in\mathbb{R}^{d}\label{eq:exponential_power_kernel}
\end{equation}
is given for $\alpha>0$ and $\mathbf{x}\in\mathbb{R}^{d}$ by
\begin{equation}
\mathcal{F}(\mathbf{x})=f(\left\Vert \mathbf{x}\right\Vert )=\frac{1}{\alpha2^{d-1}\pi^{\frac{d}{2}}}H_{1,2}^{1,1}\bigg(\frac{\left\Vert \mathbf{x}\right\Vert ^{2}}{4}\bigg|\begin{array}{l}
(1-\frac{d}{\alpha},\frac{2}{\alpha})\\
(0,1),(1-\frac{d}{2},1)
\end{array}\bigg).\label{eq:fourier_exponential_power}
\end{equation}
\end{prop}
\begin{proof}
Use equation~\eqref{eq:fourier_meijer}, the formula $e^{-s}=G_{0,1}^{1,0}\big(s\big|0\big)$,
\citep[equation~8.3.1-21]{prudnikov1990integrals} and \citep[equation~2.25.1]{prudnikov1990integrals}.
From \citep[Theorem~1.1 case~7]{mathai2010hfunction}, this $H$-function
exists and is analytical for all $\alpha>0$ and $\mathbf{x}\neq\mathbf{0}$, and the case $\mathbf{x}=\mathbf{0}$ can be verified separately.
\end{proof}
\begin{rem}
In the case $\alpha=2$, one retrieves, as expected, a multivariate
Gaussian distribution
\begin{align*}
f(\left\Vert \mathbf{x}\right\Vert ) & =\frac{1}{2^{d}\pi^{\frac{d}{2}}}H_{1,2}^{1,1}\bigg(\frac{\left\Vert \mathbf{x}\right\Vert ^{2}}{4}\bigg|\begin{array}{l}
(1-\frac{d}{2},1)\\
(0,1),(1-\frac{d}{2},1)
\end{array}\bigg)=\frac{1}{2^{d}\pi^{\frac{d}{2}}}G_{1,2}^{1,1}\bigg(\frac{\left\Vert \mathbf{x}\right\Vert ^{2}}{4}\bigg|\begin{array}{l}
1-\frac{d}{2}\\
0,1-\frac{d}{2}
\end{array}\bigg)\\
 & =\frac{1}{2^{d}\pi^{\frac{d}{2}}}e^{-\frac{\left\Vert \mathbf{x}\right\Vert ^{2}}{4}}=\frac{1}{\left(2\pi\sigma^{2}\right)^{\frac{d}{2}}}e^{-\frac{\left\Vert \mathbf{x}\right\Vert ^{2}}{2\sigma^{2}}}
\end{align*}
where each independent Gaussian component has variance $\sigma^{2}=2$.
\end{rem}

\subsection{Generalized Cauchy kernel\label{subsec:generalized_cauchy_kernel}}

A formula for the spectral density of the generalized Cauchy kernel
is provided in \citep[Proposition~4.2]{lim2010analytic} (see also
\citep[Theorem~2.1]{faouzi2020zastavnyi}) as a weighted sum of two
Fox-Wright $\Psi_{2,1}$ generalized hypergeometric functions. The
following proposition provides an equivalent but more concise formula
in terms of the Fox $H$-function.
\begin{prop}
\label{prop:fourier_generalized_cauchy}The Fourier transform of the
generalized Cauchy kernel 
\begin{equation}
K(\mathbf{u})=k(\left\Vert \mathbf{u}\right\Vert )=\frac{1}{\left(1+\lambda\left\Vert \mathbf{u}\right\Vert ^{\alpha}\right)^{\beta}}\ ,\ \mathbf{u}\in\mathbb{R}^{d}\label{eq:generalized_cauchy_kernel}
\end{equation}
where $\alpha>0$, $\beta>0$, $\lambda>0$, is given for $\mathbf{x}\in\mathbb{R}^{d}$ by
\begin{align}
\mathcal{F}(\mathbf{x})=f(\left\Vert \mathbf{x}\right\Vert ) & =\frac{1}{\alpha2^{d-1}\pi^{\frac{d}{2}}\lambda^{\frac{d}{\alpha}}\Gamma(\beta)}H_{1,3}^{2,1}\bigg(\frac{\left\Vert \mathbf{x}\right\Vert ^{2}}{4\lambda^{\frac{2}{\alpha}}}\bigg|\begin{array}{l}
(1-\frac{d}{\alpha},\frac{2}{\alpha})\\
(0,1),(\beta-\frac{d}{\alpha},\frac{2}{\alpha})(1-\frac{d}{2},1)
\end{array}\bigg).\label{eq:fourier_generalized_cauchy}
\end{align}
The case $\lambda=1$ corresponds to the classical generalized Cauchy
parameterization \citep{gneiting2004stochastic}, while the case $\lambda=\frac{1}{2\beta}$
corresponds to the one proposed in Table~\ref{tab:new_random_projections}.
\end{prop}
\begin{proof}
Use equation~\eqref{eq:fourier_meijer}, the formula $\left(1+\lambda s\right)^{-\beta}=\frac{1}{\Gamma(\beta)}G_{1,1}^{1,1}\Big(\lambda s\Big|\begin{array}{l}
{\scriptstyle 1-\beta}\\
{\scriptstyle 0}
\end{array}\Big)$, \citep[equation~8.3.1-21]{prudnikov1990integrals} and \citep[equation~2.25.1]{prudnikov1990integrals}.
From \citep[Theorem~1.1 case~1]{mathai2010hfunction}, this $H$-function
exists and is analytical for all $\alpha>0$ and $\mathbf{x}\neq\mathbf{0}$, and the case $\mathbf{x}=\mathbf{0}$ can be verified separately.
\end{proof}

\subsection{Generalized Mat\'ern kernel\label{subsec:generalized_matern_kernel}}
\begin{prop}
\label{prop:fourier_generalized_matern}The Fourier transform of the
generalized Mat\'ern kernel 
\begin{equation}
K(\mathbf{u})=k(\left\Vert \mathbf{u}\right\Vert )=\frac{(\sqrt{2\beta}\left\Vert \mathbf{u}\right\Vert ^{\frac{\alpha}{2}})^{\beta}}{\Gamma(\beta)2^{\beta-1}}\mathcal{K}_{\beta}(\sqrt{2\beta}\left\Vert \mathbf{u}\right\Vert ^{\frac{\alpha}{2}})\ ,\ \mathbf{u}\in\mathbb{R}^{d}\label{eq:generalized_matern_kernel}
\end{equation}
where $\alpha>0$, $\beta>0$, is given for $\mathbf{x}\in\mathbb{R}^{d}$ by
\begin{equation}
\mathcal{F}(\mathbf{x})=f(\left\Vert \mathbf{x}\right\Vert )=\frac{1}{\alpha2^{d-1}\pi^{\frac{d}{2}}\Gamma(\beta)(\beta/2)^{\frac{d}{\alpha}}}H_{2,2}^{1,2}\bigg(\frac{\left\Vert \mathbf{x}\right\Vert ^{2}}{4(\beta/2)^{\frac{2}{\alpha}}}\bigg|\begin{array}{l}
(1-\frac{d}{\alpha}-\beta,\frac{2}{\alpha}),(1-\frac{d}{\alpha},\frac{2}{\alpha})\\
(0,1),(1-\frac{d}{2},1)
\end{array}\bigg).\label{eq:fourier_generalized_matern}
\end{equation}
\end{prop}
\begin{proof}
Use equation~\eqref{eq:fourier_meijer}, the formula $\frac{(\sqrt{2\beta s})^{\beta}}{\Gamma(\beta)2^{\beta-1}}\mathcal{K}_{\beta}(\sqrt{2\beta s})=\frac{1}{\Gamma(\beta)}G_{0,2}^{2,0}\Big(\frac{\beta}{2}s\Big|\beta,0\Big)$,
\citep[equation~8.3.1-21]{prudnikov1990integrals} and \citep[equation~2.25.1]{prudnikov1990integrals}.
From \citep[Theorem~1.1 case~7]{mathai2010hfunction}, this $H$-function
exists and is analytical for all $\alpha>0$ and $\mathbf{x}\neq\mathbf{0}$, and the case $\mathbf{x}=\mathbf{0}$ can be verified separately.
\end{proof}
\begin{rem}
In the case $\alpha=2$, one retrieves, as expected, a multivariate
$t$-distribution 
\begin{align}
f(\left\Vert \mathbf{x}\right\Vert ) & =\frac{1}{2^{d}\pi^{\frac{d}{2}}\Gamma(\beta)(\beta/2)^{\frac{d}{2}}}H_{2,2}^{1,2}\bigg(\frac{\left\Vert \mathbf{x}\right\Vert ^{2}}{2\beta}\bigg|\begin{array}{l}
(1-\frac{d}{2}-\beta,1),(1-\frac{d}{2},1)\\
(0,1),(1-\frac{d}{2},1)
\end{array}\bigg)\nonumber \\
 & =\frac{1}{\pi^{\frac{d}{2}}\Gamma(\beta)(2\beta)^{\frac{d}{2}}}G_{2,2}^{1,2}\bigg(\frac{\left\Vert \mathbf{x}\right\Vert ^{2}}{2\beta}\bigg|\begin{array}{l}
1-\frac{d}{2}-\beta,1-\frac{d}{2}\\
0,1-\frac{d}{2}
\end{array}\bigg)\nonumber \\
 & =\frac{\Gamma(\frac{d}{2}+\beta)\Gamma(\frac{d}{2})}{\pi^{\frac{d}{2}}\Gamma(\beta)(2\beta)^{\frac{d}{2}}}{}_{2}\tilde{F}_{1}\bigg(\frac{d}{2}+\beta,\frac{d}{2};\frac{d}{2};-\frac{\left\Vert \mathbf{x}\right\Vert ^{2}}{2\beta}\bigg)\nonumber \\
 & =\frac{\Gamma(\frac{d}{2}+\beta)}{\pi^{\frac{d}{2}}\Gamma(\beta)(2\beta)^{\frac{d}{2}}}\bigg(1+\frac{\left\Vert \mathbf{x}\right\Vert ^{2}}{2\beta}\bigg)^{-\frac{d}{2}-\beta}\label{eq:multivariate_t}
\end{align}
where we used the special function formula $_{2}\tilde{F}_{1}(a,b;b;z)=(1-z)^{-a}/\Gamma(b)$.
Equation \eqref{eq:multivariate_t} is exactly the density of a multivariate
$t$-distribution with $\nu=2\beta$ degrees of freedom \citep[equation~(1.1)]{kotz2004multivariate}.
\end{rem}

\subsection{Tricomi kernel\label{subsec:tricomi_kernel}}
\begin{prop}
\label{prop:fourier_tricomi}The Fourier transform of the Tricomi
kernel 
\begin{equation}
K(\mathbf{u})=k(\left\Vert \mathbf{u}\right\Vert )=\frac{\Gamma\left(\beta+\gamma\right)}{\Gamma\left(\gamma\right)}\,\mathcal{U}\!\left(\beta,1-\gamma,\lambda\left\Vert \mathbf{u}\right\Vert ^{\alpha}\right)\ ,\ \mathbf{u}\in\mathbb{R}^{d}\label{eq:tricomi_kernel}
\end{equation}
where $\alpha>0$, $\beta>0$, $\lambda>0$, is given for $\mathbf{x}\in\mathbb{R}^{d}$ by
\begin{equation}
\mathcal{F}(\mathbf{x})=f(\left\Vert \mathbf{x}\right\Vert )=\frac{1}{\alpha2^{d-1}\pi^{\frac{d}{2}}\Gamma(\beta)\Gamma(\gamma)\lambda^{\frac{d}{\alpha}}}H_{2,3}^{2,2}\bigg(\frac{\left\Vert \mathbf{x}\right\Vert ^{2}}{4\lambda^{\frac{2}{\alpha}}}\bigg|\begin{array}{l}
(1-\frac{d}{\alpha},\frac{2}{\alpha}),(1-\frac{d}{\alpha}-\gamma,\frac{2}{\alpha})\\
(0,1),(\beta-\frac{d}{\alpha},\frac{2}{\alpha}),(1-\frac{d}{2},1)
\end{array}\bigg).\label{eq:fourier_tricomi}
\end{equation}
\end{prop}
\begin{proof}
Use equation~\eqref{eq:fourier_meijer}, the formula $\frac{\Gamma\left(\beta+\gamma\right)}{\Gamma\left(\gamma\right)}\,\mathcal{U}\!\left(\beta,1-\gamma,\lambda s\right)=\frac{1}{\Gamma(\beta)\Gamma(\gamma)}G_{1,2}^{2,1}\Big(\lambda s\Big|\begin{array}{l}
{\scriptstyle 1-\beta}\\
{\scriptstyle 0,\gamma}
\end{array}\Big)$, \citep[equation~8.3.1-21]{prudnikov1990integrals} and \citep[equation~2.25.1]{prudnikov1990integrals}.
From \citep[Theorem~1.1 case~7]{mathai2010hfunction}, this $H$-function
exists and is analytical for all $\alpha>0$ and $\mathbf{x}\neq\mathbf{0}$, and the case $\mathbf{x}=\mathbf{0}$ can be verified separately.
\end{proof}
Proposition~\ref{prop:fourier_tricomi} can be used to obtain an
analytical formula for the spectral density of the confluent hypergeometric
kernel \citep{ma2023beyond}.
\begin{cor}\label{cor:confluent_hypergeometric_spectral_density}
The spectral density of the confluent hypergeometric kernel, defined
by $K(\mathbf{u})=\frac{\Gamma\left(\beta+\gamma\right)}{\Gamma\left(\gamma\right)}\,\mathcal{U}\!\left(\beta,1-\gamma,\gamma\left\Vert \mathbf{u}\right\Vert ^{2}\right)$,
is given explicitly by
\begin{align}
f(\left\Vert \mathbf{x}\right\Vert ) & =\frac{1}{2^{d}\pi^{\frac{d}{2}}\Gamma(\beta)\Gamma(\gamma)\gamma^{\frac{d}{2}}}G_{2,3}^{2,2}\bigg(\frac{\left\Vert \mathbf{x}\right\Vert ^{2}}{4\gamma}\bigg|\begin{array}{l}
1-\frac{d}{2},1-\frac{d}{2}-\gamma\\
0,\beta-\frac{d}{2},1-\frac{d}{2}
\end{array}\bigg)\ ,\ \mathbf{x}\in\mathbb{R}^{d}.\label{eq:fourier_confluent}
\end{align}
\end{cor}
\begin{proof}
Set $\alpha=2$ and $\lambda=\gamma$ in equation~\eqref{eq:fourier_tricomi},
and simplify the result.
\end{proof}

\end{document}